%% file: root.tex
\documentclass[letterpaper,10pt,conference]{ieeeconf}
\IEEEoverridecommandlockouts
\makeatletter
\let\NAT@parse\undefined
\makeatother

\usepackage{amssymb,amsmath}
\usepackage{booktabs,multirow}
\usepackage[font=footnotesize,labelformat=simple]{subcaption}
\usepackage[font=footnotesize]{caption}

\usepackage{tikz}
\usepackage[mode=buildnew]{standalone}
\usepackage{xspace}
\usepackage{algorithm}
\usepackage[noend]{algpseudocode}
\usepackage{tabulary}
\usepackage{verbatim}
\usepackage[numbers,sort&compress]{natbib}

\usepackage{thmtools}
\usepackage{mathtools}
\usepackage{hyperref}

\hypersetup{
  bookmarksopen,
  bookmarksnumbered,
  pdfpagemode=UseOutlines,
  colorlinks=true,
  linkcolor=blue,
  anchorcolor=blue,
  citecolor=blue,
  filecolor=blue,
  menucolor=blue,
  urlcolor=blue
}

\input{macros/math_commands.tex}
\input{macros/text_shortcuts.tex}
\input{macros/prl_commands.tex}

\newtheorem{theorem}{Theorem}

\title{\LARGE \bf
  Posterior Sampling for Anytime Motion Planning \\ on Graphs with Expensive-to-Evaluate Edges
}

\author{
  Brian Hou, Sanjiban Choudhury, Gilwoo Lee, Aditya Mandalika, and Siddhartha S. Srinivasa%
  \thanks{
    This work was (partially) funded by the National Institute of Health R01 (\#R01EB019335),
    National Science Foundation CPS (\#1544797), National Science Foundation NRI (\#1637748),
    the Office of Naval Research, the RCTA, Amazon, and Honda Research Institute USA.
    Brian Hou is partially supported by a NASA Space Technology Research Fellowship.
    Gilwoo Lee is partially supported by Kwanjeong Educational Foundation.}%
  \thanks{
    All authors are with the
    Paul G. Allen School of Computer Science \& Engineering,
    University of Washington, Seattle, WA 98195
    \texttt{\{bhou, sanjibac, gilwoo, adityavk, siddh\}@cs.uw.edu}}%
}

\begin{document}

\maketitle

\input{inputs/abstract.tex}
\input{inputs/introduction.tex}
\input{inputs/related-work.tex}
\input{inputs/problem.tex}
\input{inputs/method.tex}
\input{inputs/results.tex}
\input{inputs/discussion.tex}
\input{inputs/arxiv-link.tex}

\clearpage

{
\small
\bibliographystyle{unsrtnat}
\bibliography{references}
}

\clearpage

\input{inputs/appendix.tex}

\end{document}

%% file: macros/math_commands.tex
\usepackage{amsmath,amsfonts,bm}

\def\eqref#1{equation~\ref{#1}}

\def\1{\bm{1}}

\DeclareMathAlphabet{\mathsfit}{\encodingdefault}{\sfdefault}{m}{sl}
\SetMathAlphabet{\mathsfit}{bold}{\encodingdefault}{\sfdefault}{bx}{n}

\DeclareMathOperator*{\argmin}{arg\,min}

\newcommand{\pair}[2]{\left( #1, #2\right)}
\newcommand{\set}[1]{\left\lbrace #1\right\rbrace}
\newcommand{\seq}[2]{\left(#1_{1}, #1_{2}, \ldots, #1_{#2}\right)}

\newcommand{\real}[0]{\mathbb{R}}
\newcommand{\expect}[2]{\mathbb{E}_{#1}\left[#2\right]}
\newcommand{\Ind}[0]{\mathbb{I}}

\newcommand{\graph}[0]{G}
\newcommand{\vertexSet}[0]{V}
\newcommand{\edgeSet}[0]{E}
\newcommand{\vertex}[0]{v}
\newcommand{\edge}[0]{e}
\newcommand{\start}[0]{v_s}
\newcommand{\goal}[0]{v_g}

\newcommand{\weight}[0]{w}
\newcommand{\weightMax}[0]{C_{\rm max}}

\renewcommand{\path}[0]{\xi}

\newcommand{\world}[0]{\phi}
\newcommand{\history}[0]{\psi}

\newcommand{\mdp}[0]{\mathcal{M}}
\newcommand{\state}[0]{s}
\newcommand{\stateSet}[0]{\mathcal{S}}
\newcommand{\action}[0]{a}
\newcommand{\actionSet}[0]{\mathcal{A}}
\newcommand{\reward}[0]{R}
\newcommand{\transition}[0]{T}
\newcommand{\horizon}[0]{\tau}
\newcommand{\policy}[0]{\mu}

\newcommand{\mdpSet}[0]{ \mathbf{M} }

\newcommand{\regret}[0]{\textsc{Regret}}
\newcommand{\evalSet}[0]{\edgeSet_{\text{eval}}}

%% file: macros/text_shortcuts.tex
\newcommand{\algName}[0]{\textsc{PSMP}\xspace}
\newcommand{\algFullName}[0]{Posterior Sampling for Motion Planning\xspace}
\newcommand{\lazysp}[0]{\textsc{LazySP}\xspace}
\newcommand{\bisect}[0]{\textsc{BiSECt}\xspace}
\newcommand{\direct}[0]{\textsc{DiRECt}\xspace}
\newcommand{\maxprob}[0]{\textsc{MaxProb}\xspace}

\newcommand{\frameworkFullName}[0]{Experienced Lazy Path Search\xspace}

\newcommand{\pathSelect}[0]{\texttt{ComputePath}\xspace}

\newcommand{\failFast}[0]{\texttt{FailFast}\xspace}

\newcommand{\psmpBound}[0]{$\tilde{O}(\sqrt{\stateSet \actionSet T})$\xspace}

%% file: macros/prl_commands.tex
\newcommand{\xxnote}[3]{}
\ifx\hidenotes\undefined
  \renewcommand{\xxnote}[3]{\color{#2}{#1: #3}}
  \definecolor{SeaGreen}{HTML}{3FBC9D}
  \definecolor{Orange}{HTML}{D95F02}
  \definecolor{RacingGreen}{HTML}{004225}
\fi

\newcommand{\aref}[1]{Algorithm~\ref{#1}} 
\newcommand{\sref}[1]{Section~\ref{#1}} 
\newcommand{\fref}[1]{Fig.~\ref{#1}} 
\newcommand{\tabref}[1]{Table~\ref{#1}} 
\newcommand{\thmref}[1]{Theorem~\ref{#1}} 
\newcommand{\appref}[1]{Appendix~\ref{#1}} 
\newcommand{\extappref}[1]{\href{https://arxiv.org/abs/2002.11853}{Appendix~\ref*{#1}}} 
\renewcommand{\extappref}[1]{\appref{#1}} 

\newcommand{\fullFigGap}[0]{\vspace{-1.5\baselineskip}}

%% file: inputs/abstract.tex

\begin{abstract}
Collision checking is a computational bottleneck in motion planning, requiring lazy algorithms that explicitly reason about when to perform this computation.
Optimism in the face of collision uncertainty minimizes the number of checks before finding the shortest path.
However, this may take a prohibitively long time to compute, with no other feasible paths discovered during this period.
For many real-time applications, we instead demand strong anytime performance, defined as minimizing the cumulative lengths of the feasible paths yielded over time.
We introduce \algFullName (\algName), an anytime lazy motion planning algorithm that leverages learned posteriors on edge collisions to quickly discover an initial feasible path and progressively yield shorter paths.
\algName obtains an expected regret bound of \psmpBound and outperforms comparative baselines on a set of 2D and 7D planning problems. 
\end{abstract}

%% file: inputs/introduction.tex

\section{Introduction}
\label{sec:introduction}

We formalize the problem of anytime motion planning.
Existing algorithms typically make \emph{asymptotic} guarantees~\citep{karaman2011sampling} that they will eventually find the optimal path.
However, such analysis leaves several practical questions unanswered.
Given a budget of computation time, how sub-optimal will the resulting path be?
How will increasing the computation budget improve the quality of the solution?
Formalizing these questions helps us better understand important \emph{anytime} properties, not just asymptotic properties.
This will also enable practitioners to make more informed choices about the algorithms they deploy.

We focus on anytime planning on fixed graphs.\footnotemark{}
Here, vertices are sampled robot configurations and edges are potential robot motions.
Evaluating if an edge is in collision is \emph{computationally expensive}~\citep{hauser2015lazy}.
Hence, search algorithms must be lazy~\citep{dellin2016lazysp}, i.e., minimize edge evaluation as they search for paths.
Our goal is to quickly find feasible paths and shorten them as time permits---we refer to this as lazy anytime search~\citep{likhachev2004ara}.

What if such an algorithm was provided a posterior distribution of edge collisions?  This could either be based on a dataset of prior experience or domain knowledge about obstacle geometries. The search must consider two factors: the length of a path and the likelihood of it being in collision.
A desirable outcome, shown in \fref{fig:teaser}, is to initially evaluate longer paths that have lower probability of collision.
Eventually, as uncertainty collapses, the search evaluates shorter and shorter paths.
This strategy encapsulates a fundamental trade-off: it can either explore shorter paths to potentially improve future performance or exploit the most likely path to attain better immediate performance.

We formalize this within the framework of Bayesian Reinforcement Learning (BRL).
We first define lazy search on a graph as solving a deterministic, goal-directed, Markov Decision Process (MDP) where rewards (collision status) are unknown.
A BRL algorithm explores this MDP as it attempts to find the optimal policy (path).
To judge how quickly an algorithm learns, we consider the bandit setting~\citep{bubeck2012regret}:
in each round of learning, an agent pulls an arm (evaluates a path), receives a loss (negative of path length), and accumulates regret with respect to the optimal arm.
A low expected regret~\citep{hannan1957approximation} corresponds to evaluating edges that not only lead to shorter paths, but also drive down uncertainty over time.
Hence, our key insight is:
\begin{quote}
Good anytime search performance is equivalent to minimizing \emph{Bayesian regret}.
\end{quote}

However, the space of paths is combinatorially large, which makes many bandit algorithms that require explicit posteriors inapplicable.
Fortunately, while explicitly computing this posterior is hard, sampling from it is quite easy!
Posterior sampling offers strong guarantees on Bayesian regret~\citep{osband2013more}.
Our algorithm, \algFullName (\algName), samples a graph from the posterior and only evaluates edges along the shortest path in that graph.
It is both simple to implement and---given a posterior to sample from---free of tuning parameters. We make the following contributions.

\begin{itemize}
\item We introduce a novel formulation of anytime search on graphs as an instance of Bayesian Reinforcement Learning (\sref{sec:problem-statement}).
\item We introduce a general framework, \frameworkFullName, that unifies several existing search algorithms that leverage prior experience (\sref{sec:method}).
\item We show that \algName has good theoretical anytime performance by bounding its Bayesian regret (\sref{sec:method}).
\item We demonstrate that \algName effectively leverages posteriors to outperform comparative baselines on a set of 2D and 7D motion planning problems (\sref{sec:results}).
\end{itemize}

\footnotetext{
  While analyzing anytime algorithms that continue to sample the configuration space is the eventual goal, obtaining meaningful bounds requires analyzing the nature of probability distributions over continuous configuration space geometry.
  This becomes quite challenging, even for simple geometries, and hence is currently out of scope.
}

\begin{figure*}[!t]
  \centering
  \includegraphics[width=0.69\linewidth]{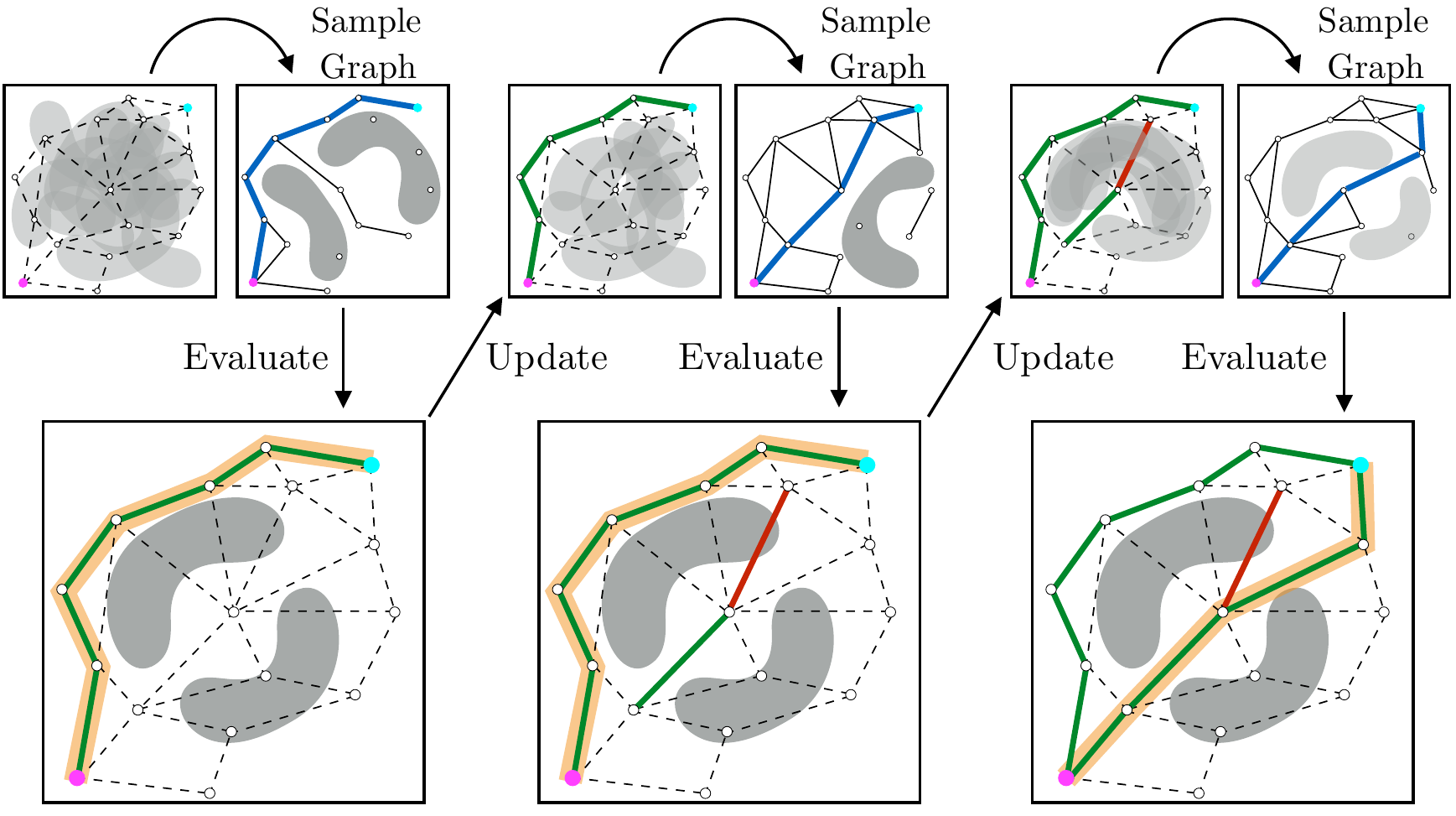}
  \caption{
    \algFullName is a Bayesian anytime motion planning algorithm.
    The graph's edge collision statuses are initially unknown (dashed).
    In each iteration of \algName, the posterior (top left) is sampled to produce a graph (top right).
    The sampled graph's shortest path (blue) is evaluated for collisions against the real world (bottom).
    Edges are either found to be in collision (red) or collision-free (green);
    those statuses are used to update the posterior.
    If all edges in the proposed path are collision-free, \algName updates its current shortest path (yellow), which can be emitted at any time.
    \fullFigGap
  }
  \label{fig:teaser}
\end{figure*}

%% file: inputs/related-work.tex

\section{Related Work}
\label{sec:related-work}

\subsection{Priors in Lazy Search}

Planning with expensive collision-checking is a well-studied problem in motion planning.
Lazy search approaches deal with this by only checking edges that lie along the estimated shortest path~\citep{bohlin2000path,dellin2016lazysp} or the shortest subpath~\citep{cohen2015planning,mandalika2018lrastar}.
For real-world robotics problems, leveraging priors on edge collisions can produce significant speed-ups.
FuzzyPRM~\citep{nielsen2000two} evaluates paths that are most likely to be feasible.
GLS~\citep{mandalika2019gls} uses priors to quickly invalidate subpaths until the shortest path is found.
STROLL~\citep{bhardwaj2019stroll} learns an edge evaluation policy for LazySP.
\bisect~\cite{choudhury2017active} and \direct~\cite{choudhury2018bayesian} formalize Bayesian motion planning and compute near Bayes-optimal policies for finding feasible paths.
However, these approaches do not aim for anytime performance.

Several methods model collision posteriors to exploiting structure in planning.
One approach is to predict validity of unevaluated edges given the outcomes of evaluated edges~\cite{esposito2016matrix}.
Other approaches try to model the configuration space belief given observed collisions and guide search with that belief~\citep{burns2005sampling,pan2012ibl,huh2016learning,lacevic2016burs}.
However, these approaches do not directly aim to approximate the Bayesian posterior.

\subsection{Anytime Planning}

For many real-time planning applications, an algorithm must be able to deal with an unknown planning time budget.
This is achieved by incremental sampling methods, such as RRT*~\citep{karaman2011sampling} or RRT++~\citep{abbasi2010extending}, which guarantee asymptotic optimality.
However, they make no promises on convergence rate and are often slow in practice.
Incremental densification techniques, on the other hand, offer provable speed-ups by restricting new samples to a region that can only improve the current solution~\citep{ferguson2006anytime,gammell2014informed,gammell2015batch}. However, these methods cannot provably exploit priors on the configuration space.

Another way of viewing anytime planning is through the lens of heuristic search on large graphs.
Weighted A* search with an inflated heuristic finds feasible paths quickly, although the solution may be suboptimal.
Anytime variants of A*~\citep{likhachev2004ara,van2011anytime} efficiently run a succession of weighted A* searches with decreasing inflation.
However, heuristics may not always indicate existence of feasible paths.
POMP~\citep{choudhury2016pomp} uses priors on edge validity to explicitly trade-off path likelihood and path length.
AEE*~\citep{narayanan2017heuristic} uses Bernoulli priors on edges to generate a set of plausible shortest paths, which is then evaluated in an anytime fashion.
However, these do not offer guarantees for arbitrary priors.

\subsection{Bayesian Reinforcement Learning (BRL)}

Standard RL approaches consider optimal exploration of an unknown MDP until an optimal policy is computed.
In the absence of prior knowledge, PAC-MDP~\citep{strehl2009reinforcement} approaches result in exhaustive experimentation in every possible state.
BRL~\citep{ghavamzadeh2015bayesian} introduces a prior on rewards and transitions, requiring only enough exploration to find a good policy in expectation.
Since Bayes-optimality is intractable, this can only be solved approximately~\citep{kolter2009near,chen2016pomdp}.
An alternative to Bayes-optimality is Bayesian regret, which views the learning as an online process of interacting with MDP. 
This leads to simpler algorithms such as UCRL2~\citep{jaksch2010near} and Posterior Sampling RL~\citep{osband2013more}.
We build on \citep{osband2013more} to bound Bayesian regret for the problem of anytime planning.

%% file: inputs/problem.tex

\section{Bayesian Anytime Motion Planning}
\label{sec:problem-statement}

We assume a fixed explicit graph $\graph = \set{\vertexSet, \edgeSet}$,  where $\vertexSet$ denotes a set of vertices and $\edgeSet$ a set of edges.

Given start and goal $\pair{\start}{\goal} \in \vertexSet$, a path $\path$ is a sequence of connected vertices $\seq{\vertex}{l}$, $\vertex_1 = \start, \vertex_l = \goal$.
Let $\weight: \edgeSet \rightarrow \real^{+}$ be the weight of an edge.
The length of a path is the sum of edge weights, i.e. $\weight(\path) = \sum_{\edge \in \path} \weight(\edge)$.
We define a world $\world \in {\real^{+}}^{|\edgeSet|}$ as the vector of edge weights.
The weights are unknown and discovered by edge evaluation, which is \emph{computationally expensive}.
Hence, an algorithm's planning time is determined by the edges that it evaluates.

As a planning algorithm evaluates edges $\edge_1, \cdots, \edge_N$, it uncovers a series of progressively shorter paths $\path_1, \cdots, \path_N$.
The objective of anytime planning is to minimize the cumulative length of paths, i.e., $\sum_{i=1}^{N} \weight(\path_i)$.

We consider a Bayesian setting where we have a prior distribution on worlds $P(\world)$ obtained from past experience.
As an algorithm evaluates edges, let $\history_t = \{\weight(\edge_1), \cdots, \weight(\edge_t)\}$ be the history of observations, i.e., outcome of edge evaluations.
Given this history, a Bayesian planning algorithm can compute a posterior $P(\world | \history_t)$ to decide which path to evaluate.
The objective of a \emph{Bayesian anytime planning} algorithm is to minimize the expected cumulative length of paths computed given the prior over the worlds $P(\world)$.

We will now establish an equivalence between our problem and the repeated episodic BRL problem described in \cite{osband2013more}.
We consider a deterministic finite horizon MDP $\mdp = \langle \stateSet, \actionSet, \reward^\mdp, \transition, \horizon, s_1 \rangle$.
A state $\state \in \stateSet$ corresponds to a vertex $\vertex \in \vertexSet$, actions $\actionSet(\state)$ correspond to the set of adjacent edges, and transition function $\state' = T(\state, \action)$ is the adjacency matrix of the graph.
The reward function $\reward^\mdp(\state, \action)$ is $0$ if $\state = \goal$, else it is $-\weight(\edge)$, where $\edge$ is the edge associated with $\action$.
The horizon $\horizon$ corresponds to the maximum number of edges in a path.
The initial state $\state_1$ is $\start$.

The solution to the MDP is a partial policy $\policy : \stateSet \rightarrow \actionSet$ corresponding to a path $\path$.
The policy's value is $V^{\mdp}_{\policy}(\state_1) = \sum_{j=1}^{\horizon} \reward^\mdp (\state_j, \policy(\state_j))$ is the negative path length $-\weight(\path)$.

For an unknown MDP $\mdp^*$, the reward function $\reward^{\mdp^*}$ is unknown. The prior over worlds $P(\world)$ maps to a prior $P(\reward^\mdp)$.
A learning algorithm must infer the reward function by repeatedly interacting with $\mdp^*$.
In each episode $i = 1, \cdots, m$, it executes policies $\policy_1, \cdots, \policy_m$, updates history $\history_i$, and tracks the best discovered policy $\hat{\policy}_i = \max_{j = 1, \dots, i} V^{\mdp^*}_{\policy_j}(\state_1)$.
We define an algorithm's regret to be the cumulative difference between the value of the optimal policy and the best discovered policy after each episode
\begin{equation*}
\textstyle
\regret(m) = \sum_{k=1}^m \Delta_k, \text{where } \Delta_k = V^{\mdp^*}_{\policy^*}(\state_1) - V^{\mdp^*}_{\hat{\policy}_k}(\state_1).
\end{equation*}
The objective of the BRL problem in \citep{osband2013more} is to minimize the expected regret $\expect{P(\reward^\mdp)}{\regret(m)}$, also known as the Bayesian regret $\expect{}{\regret(m)}$.
Note that this is a constant offset from the objective we defined in Bayesian anytime planning, i.e., $\expect{P(\reward^\mdp)}{\regret(m)} = \expect{P(\world)}{  \sum_{k=1}^m \weight(\path_k) - \weight(\path^*)}$. 

Although we formulate the problem more generally, we focus on a specific instantiation where each edge has a binary collision status.
Edge evaluation corresponds to collision checking the edge.
If an edge is not in collision, the weight $\weight(\edge)$ is the distance between the two vertices, which is known.
If the edge is in collision, $\weight(\edge)$ is set to a large value $\weightMax$.
For compactness, we redefine a world to be the collision status of all edges $\world \in \{0, 1\}^{|\edgeSet|}$, effectively binarizing the problem.
A feasible path has $\world(\edge) = 1, \forall \edge \in \path$.

The MDP we have defined allows us to establish equivalences between RL and other lazy motion planning formulations proposed in previous work.
The lazy \emph{shortest path problem}~\citep{dellin2016lazysp}, where the shortest feasible path must be found while eliminating all shorter paths, is equivalent to the PAC-MDP~\citep{strehl2009reinforcement} problem of optimally exploring an MDP until an optimal policy is found.
Similarly, the Bayesian version of this problem is equivalent to PAC-BAMDP~\citep{ghavamzadeh2015bayesian}.
The \emph{feasible path problem}~\citep{choudhury2017active} is equivalent to Bayes-optimally exploring the MDP until a valid policy is found.

%% file: inputs/method.tex

\section{Experienced Lazy Path Search}
\label{sec:method}

We present a general framework for experienced lazy search that uses priors on edge validities to minimize collision checking. 
This unifies search for anytime planning, as well as other objectives such as efficiently finding the shortest path or any feasible path. 
We then introduce \emph{\algFullName} (\algName), a new algorithm that bounds expected anytime planning performance.

\subsection{\frameworkFullName}
\label{subsec:method-framework}

We begin by presenting a framework for lazy search algorithms that uses priors, thus unifying several previous works in this area~\citep{dellin2016lazysp,choudhury2016pomp,choudhury2017active,bhardwaj2019stroll}.
In \frameworkFullName (\aref{algo:framework}), a \emph{proposer} lazily computes a path from the start to goal (without any edge evaluation) and a \emph{path validator} chooses edges along the path to evaluate.\footnote{Note this unifying framework differs from the framework in Generalized Lazy Search (GLS)~\citep{mandalika2019gls}. First, GLS looks at problems where planning time depends on \emph{both} graph operations and edge evaluations. Hence, it argues for {\it interleaving} search with evaluation of \emph{sub-paths}. Second, GLS exclusively considers the shortest path problem.}

\begin{table*}[!t]
\centering
\small
\caption{Different algorithms as instantiations of \frameworkFullName.}
\begin{tabulary}{\textwidth}{LCC}
\toprule
{\bf Algorithm}	&  $\pathSelect(G, P(\world | \history))$ & {\bf Performance Guarantee} \\ \midrule
PSMP ({\bf ours})                     & Sample a world $\world \sim P(\world | \history)$, return $\path^*$                           & Anytime (Bayesian regret)     \\
\lazysp~\citep{dellin2016lazysp}      & Generate optimistic world $\world$, return $\path^*$                                          & Shortest path (OFU)           \\
\maxprob~\citep{choudhury2017active}  & Set weights $- \log P(\world(\edge)=1 | \history)$, return $\path^*$                          & Feasible path (Bayes-optimal) \\ 
POMP~\citep{choudhury2016pomp}        & Set weights $\alpha w(e) - (1 - \alpha) \log P(\world(\edge)=1 | \history)$, return $\path^*$ & Anytime (Pareto optimality)   \\
\bottomrule
\end{tabulary}
\label{tab:algorithm}
\vspace{-1\baselineskip}
\end{table*}

\begin{algorithm}[!t]
\caption{\frameworkFullName}
\label{algo:framework}
\begin{algorithmic}[1]
\Require Graph $\graph$, Prior $P(\world)$, Proposer $\pathSelect(\cdot)$
\Statex
\State Initialize history $\history \gets \emptyset$, evaluated edges $\evalSet \gets \emptyset$
\While{termination criteria not met}
  \State Invoke proposer $\path = \pathSelect(G, P(\world | \history))$.
  \While{path $\path$ is not invalid \textbf{and} $\path$ is unevaluated}
    \State Evaluate unevaluated edge with highest \\
    \hskip\algorithmicindent\hskip\algorithmicindent posterior collision probability \\
    \hskip\algorithmicindent\hskip\algorithmicindent $\edge^* = \argmin_{\edge \in \path \setminus \evalSet} P(\world(\edge)=1 | \history)$.
    \State Add edge to evaluated set $\evalSet \gets \evalSet \cup \{\edge^*\}$.
  \EndWhile
  \State Update history $\history$ with outcomes.
  \State \algorithmicif\ $\path$ is valid\ \algorithmicthen\ emit $\path$.
\EndWhile
\end{algorithmic}
\end{algorithm}

We fix the path validator to the \failFast rule~\citep{mandalika2019gls} for all proposers.
This rule tries to invalidate a proposed path as quickly as possible, formally stated as follows:
\begin{theorem}
\label{thm:fail_fast}
The \failFast validator repeatedly evaluates the edge with highest probability of collision, until one edge is found to be in collision or all edges are found to be collision-free.
This is optimal for eliminating a single candidate path\footnote{However, it ignores overlap among paths unlike \cite{choudhury2018bayesian} for simplicity.}, if prior $P(\phi)$ is independent Bernoulli.
For general priors, this is near-optimal with a factor of 4.
\end{theorem}
\begin{proof}
This can be mapped to a Bayesian search problem where the goal is to sequentially search for an item (invalid edge) in a set of boxes (unevaluated edges) while minimizing cost of search. Bounds follow from \citep{dor1998optimal}. 
\end{proof}

By varying the proposer $\pathSelect(\cdot)$, we can recover several algorithms from the literature that aim for different performance guarantees. 
All proposers listed in \tabref{tab:algorithm} view $G$ with modified edge weights and propose the shortest path on the modified graph.
\lazysp~\citep{dellin2016lazysp} returns the optimistic shortest path by modifying all unevaluated edges in $G$ to be feasible.
It eliminates candidate paths in order of increasing length and terminates after a feasible path is found, yielding an OFU-like guarantee of finding the shortest path with minimal evaluations.
\maxprob~\citep{choudhury2017active} returns the most likely feasible path by modifying the weights to be negative log likelihood of validity. 
It sequentially evaluates the most probable path, terminating after a feasible path is found.
Similar to \thmref{thm:fail_fast}, \maxprob Bayes-optimally proposes the fewest paths before finding a feasible path if prior $P(\phi)$ is independent Bernoulli and near Bayes-optimally (factor of $4$) otherwise.
Finally, POMP~\citep{choudhury2016pomp} balances edge weight with the likelihood of being collision free.
Increasing $\alpha$ between iterations of \aref{algo:framework} traces out the Pareto frontier of the two objectives, starting with the most probable path while guaranteeing asymptotic optimality. However, this anytime property comes without guarantees on rate of improvement.

\begin{figure}[!t]
\centering
\frame{\includegraphics[width=0.15\linewidth]{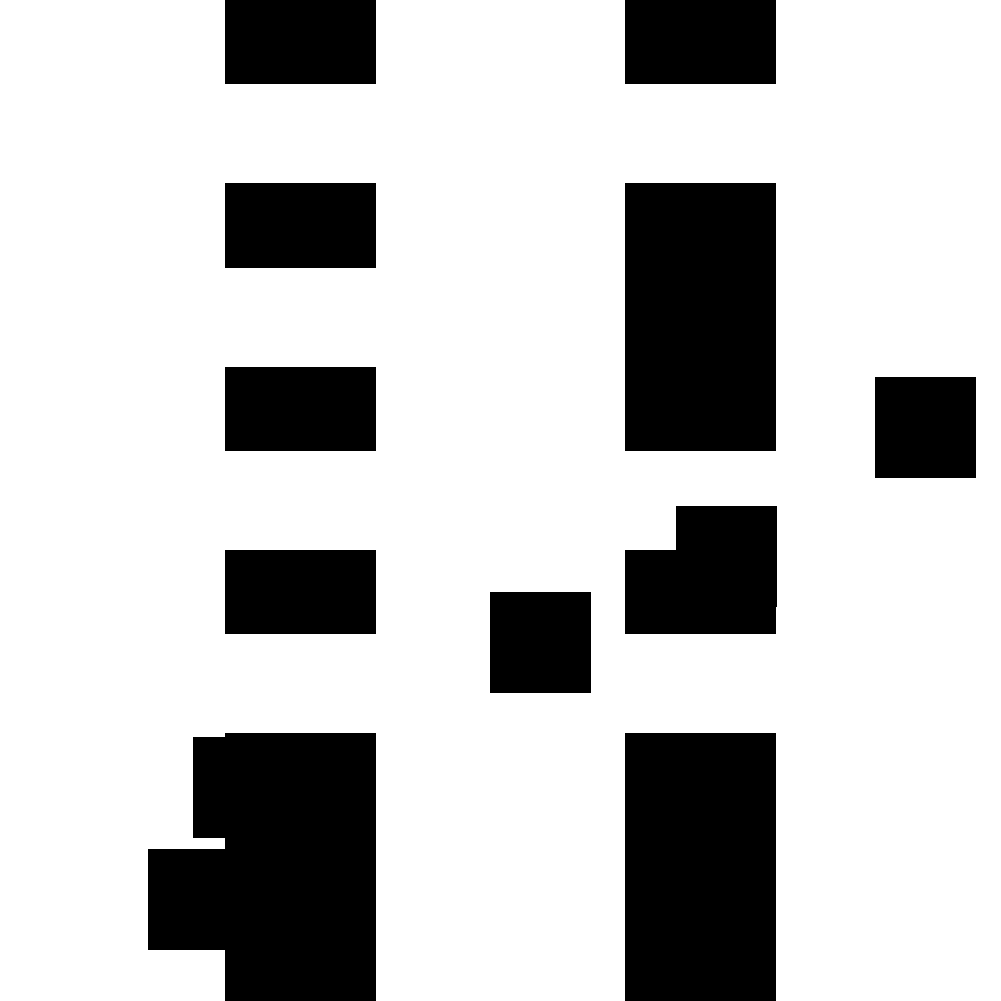}}
\frame{\includegraphics[width=0.15\linewidth]{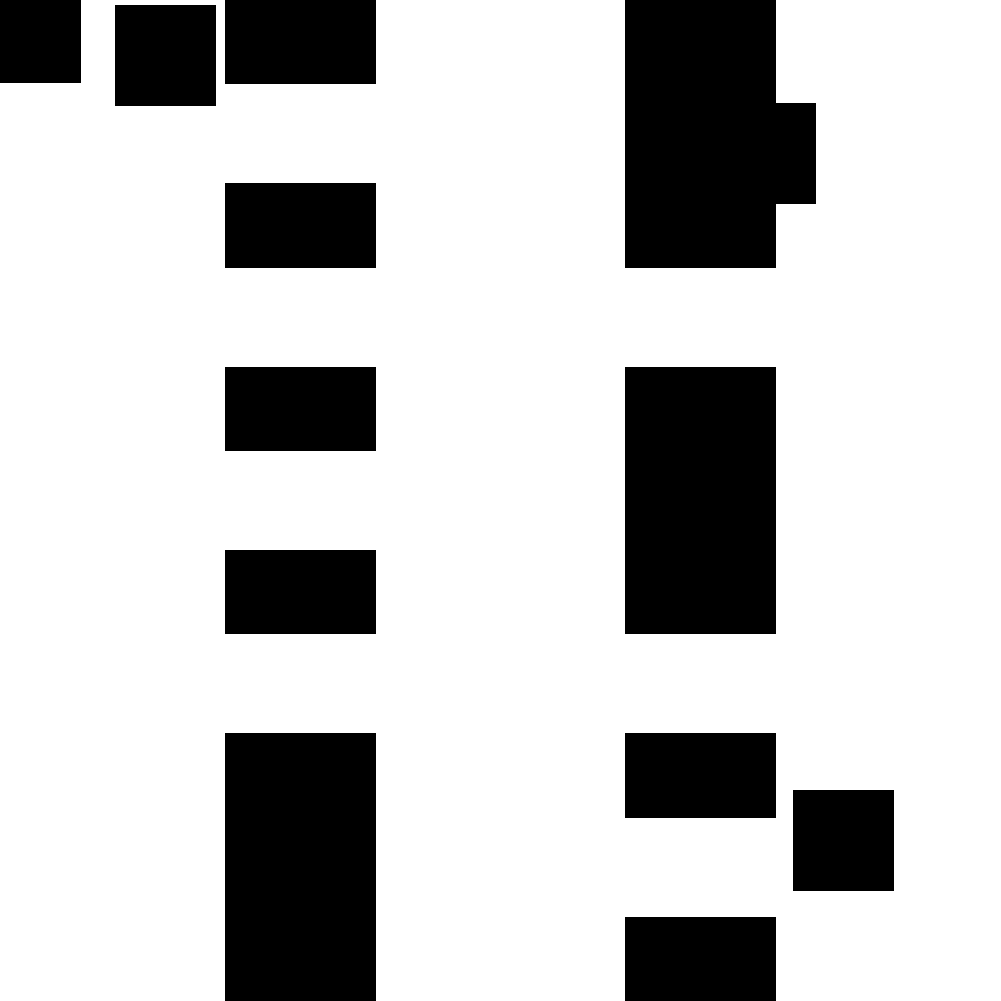}}
\frame{\includegraphics[width=0.15\linewidth]{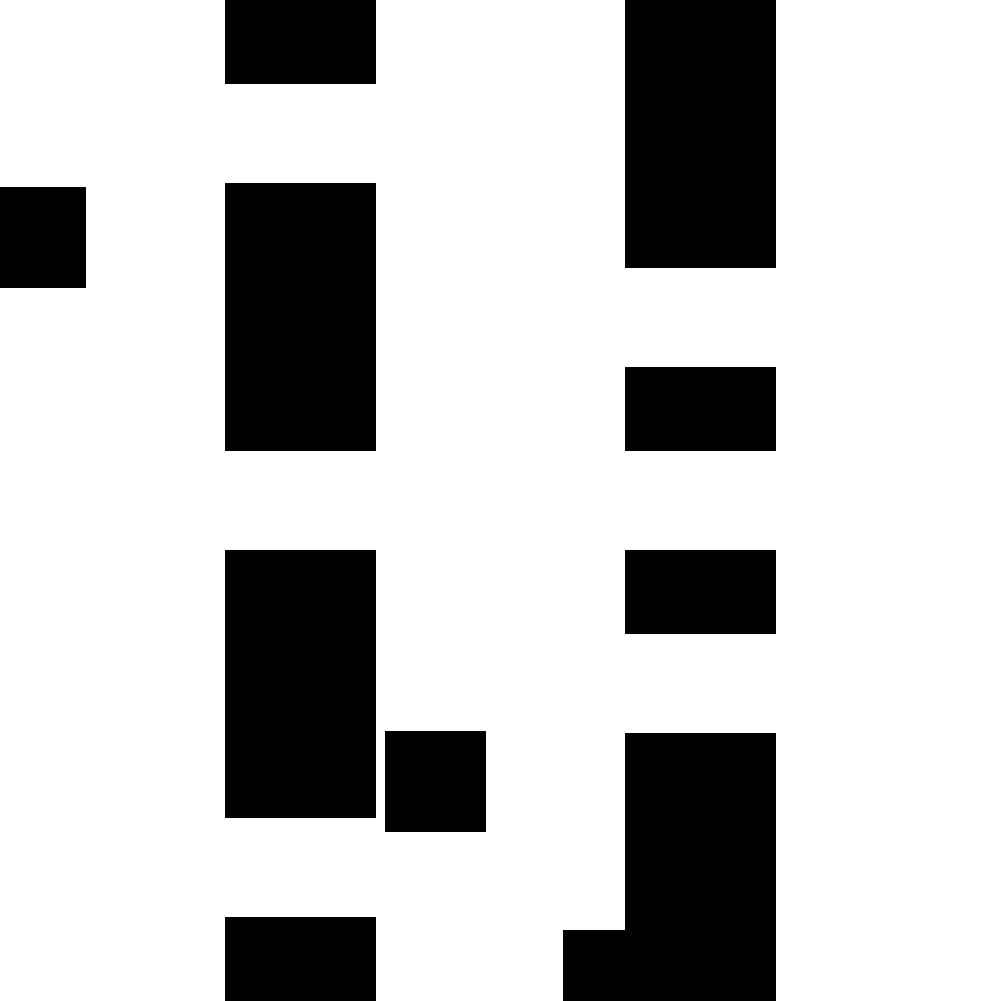}}
\hfill
\frame{\includegraphics[width=0.15\linewidth]{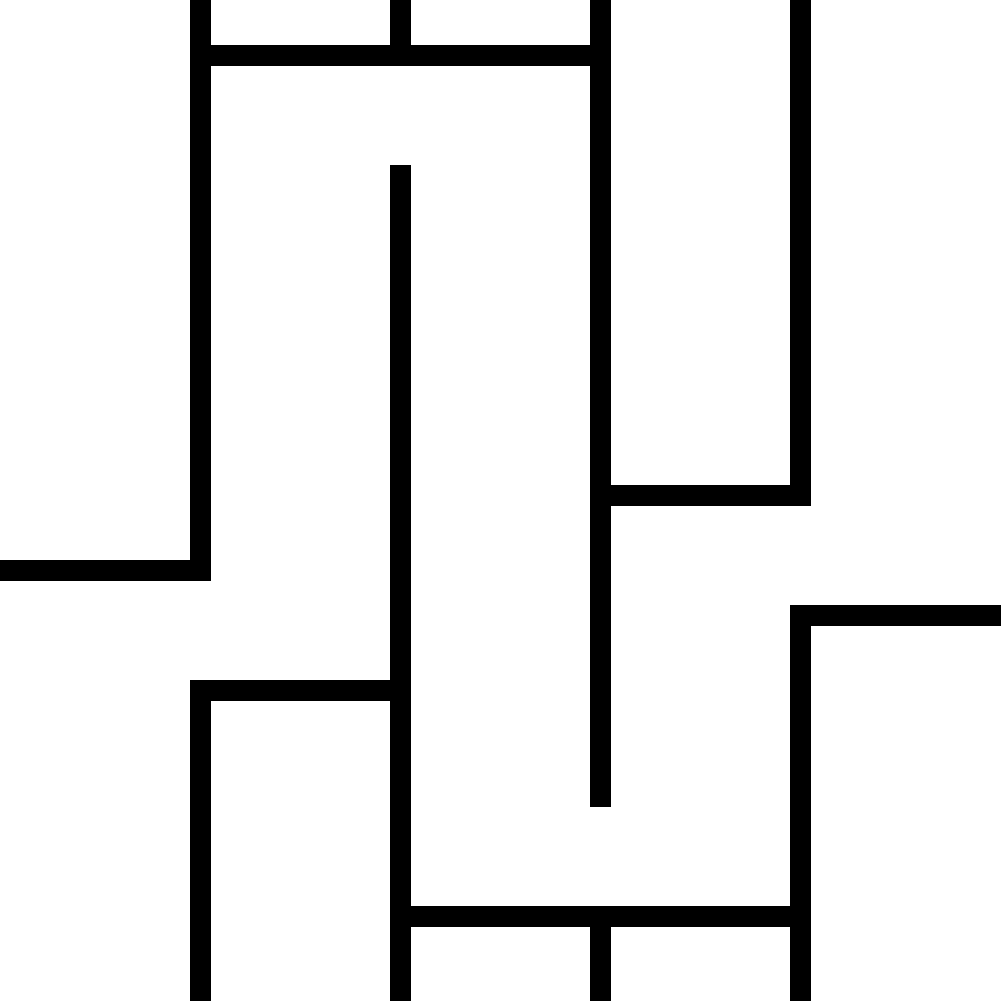}}
\frame{\includegraphics[width=0.15\linewidth]{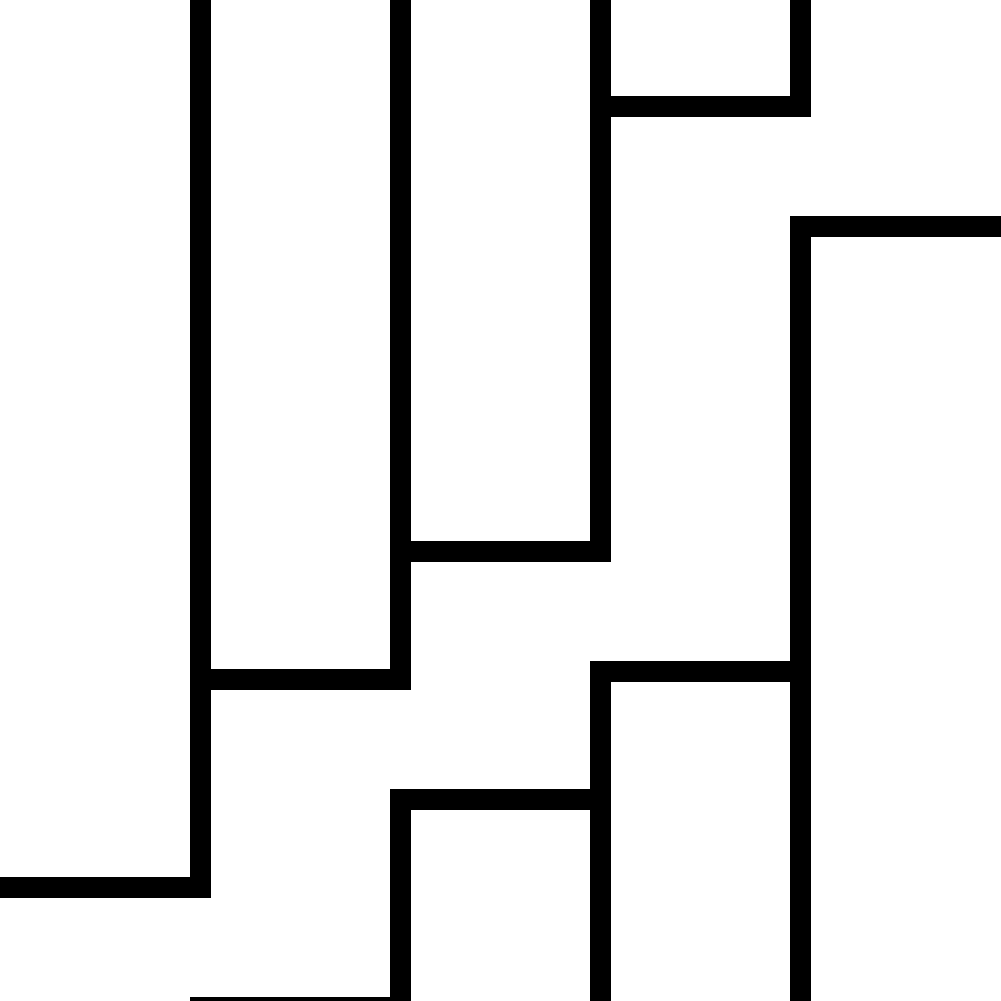}}
\frame{\includegraphics[width=0.15\linewidth]{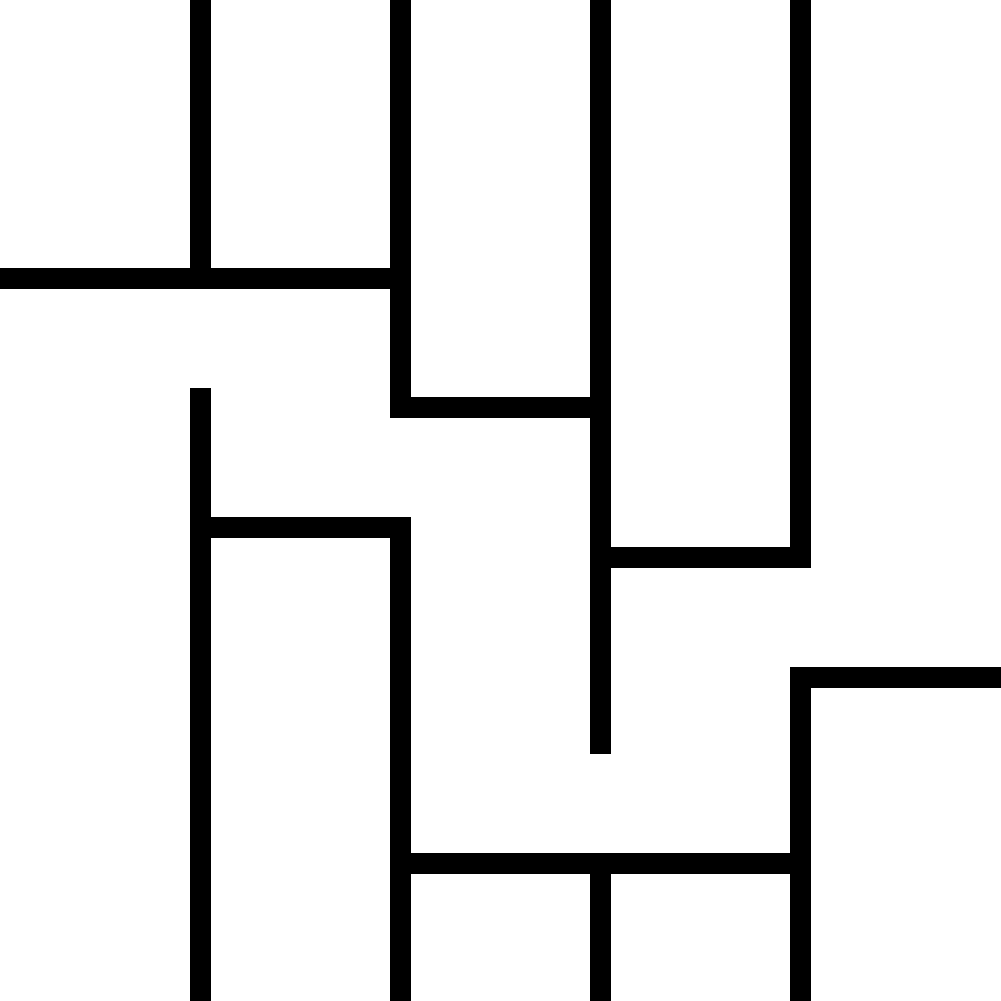}}

\medskip

\includegraphics[trim={80 0 20 0},clip,width=0.3\linewidth]{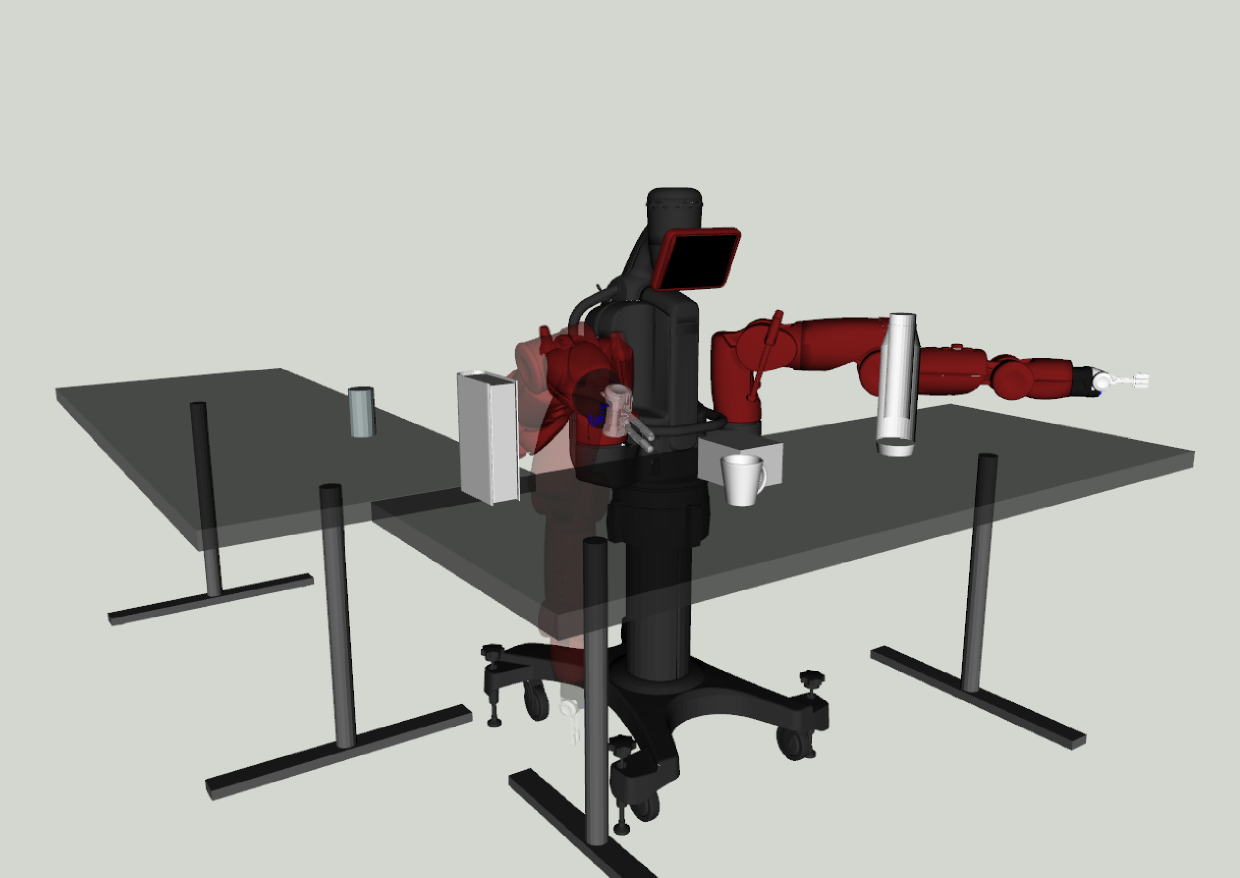}
\hfill
\includegraphics[trim={80 0 20 0},clip,width=0.3\linewidth]{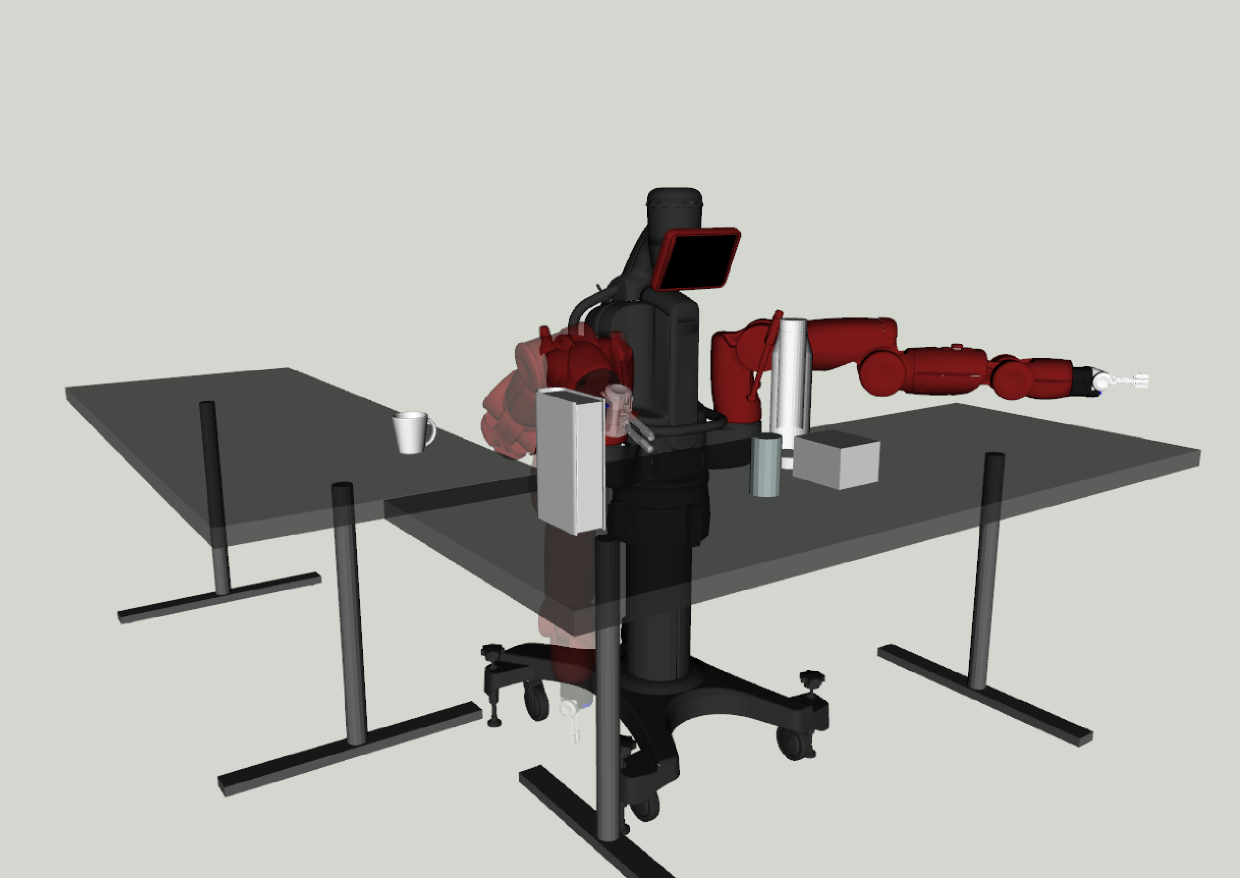}
\hfill
\includegraphics[trim={80 0 20 0},clip,width=0.3\linewidth]{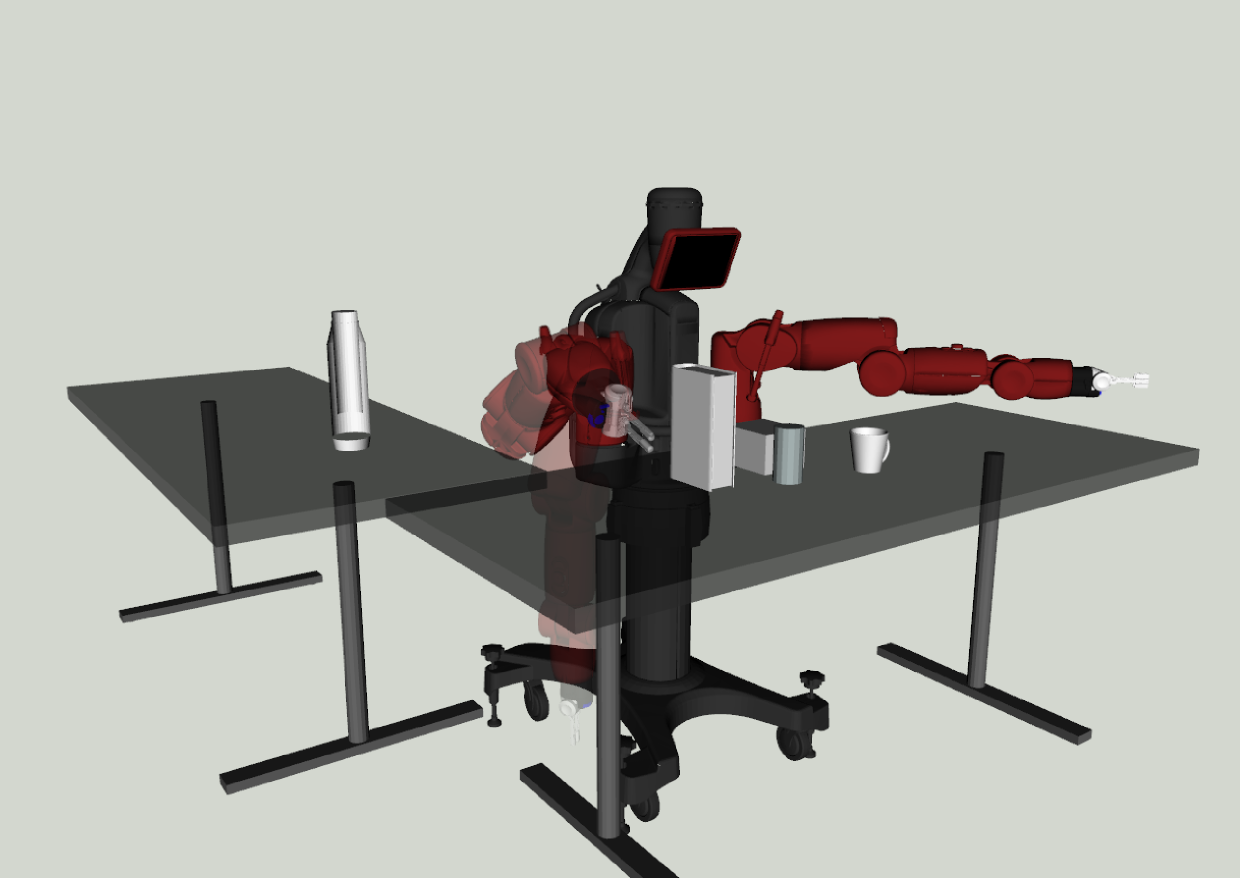}

\caption{
  Example problems from two 2-DOF datasets (top) and the 7-DOF dataset (bottom).
  In the 2-DOF datasets, a point robot navigates from the bottom left corner to the top right.
  In the 7-DOF dataset, the robot arm moves from below the table (transparent) to above, while avoiding clutter.
}
\label{fig:dataset}
\vspace{-1\baselineskip}
\end{figure}

\subsection{Posterior Sampling for Motion Planning (PSMP)}
\label{subsec:method-psmp}

PSMP aims to guarantee anytime behavior. It essentially borrows the idea of posterior sampling~\citep{osband2013more}, or Thompson sampling~\citep{thompson1933likelihood}, and applies it in the space of paths. PSMP samples a world $\phi$ from the posterior distribution $P(\phi | \psi)$ conditioned on the history $\psi$. 
It then computes the shortest path $\path^*$ on $\phi$ and proposes it for evaluation. 

We will now establish Bayesian regret bounds for PSMP, following the analysis of posterior sampling for reinforcement learning~\citep{osband2013more} and multi-armed bandits~\citep{russo2014learning}.
From an algorithmic perspective, PSMP is attractive because it requires solving only a single shortest path problem.
By contrast, other Bayesian search methods like Monte Carlo Tree Search~\citep{guez2012efficient} or even heuristics like QMDP~\citep{Littman95learningpolicies} require several calls to the search. PSMP also requires no tuning parameters.

By sampling paths according to the posterior probability they are optimal, PSMP continues to sample plausible shortest paths.
As PSMP gains more information, the posterior concentrates around the true world.
The regret for PSMP grows sublinearly as \psmpBound where $T$ is the total number of timesteps, matching the lower bound from~\citep{jaksch2010near}.
For the analysis, we assume edge weights are normalized $[0, 1]$.
\begin{theorem}
\label{thm:regret}
The expected regret is bounded as
\begin{equation}
\expect{}{\regret(T)} = O(\horizon \sqrt{\stateSet \actionSet T \log(\stateSet \actionSet T) })	
\end{equation}
\end{theorem}

\begin{proof}
We follow the analysis of \cite{osband2013more}, adapted for the special case of a deterministic MDP to obtain tighter regret bounds.
We refer the reader to \extappref{sec:regret-proof} for details.
\end{proof}

\subsection{Estimating Edge Collision Posteriors}
\label{subsec:method-posterior}

In our experiments, we consider two possible approaches for estimating the posterior distribution $P(\world | \history)$.
If there is no dataset of previous planning problems to learn from, the collision-checked configurations from $\history$ can inform a nearest neighbor-based posterior for the current problem~\cite{pan2012ibl}.
We find that only considering the 1-nearest neighbor $q_{near}$ produces the best collision estimates due to massive label imbalance in favor of collision-free points.

We differ from \cite{pan2012ibl} by assuming a uniform $\text{Beta}(1, 1)$ prior on configuration space collision probability.
The status of the nearest neighbor counts as a partial success or failure with weight $\exp(-\eta \|q - q_{near}\|)$.
The expected posterior probability that configuration $q$ is free is then
\begin{equation*}
\textstyle
\expect{}{P(\world(q) = 0 | \history)} = \frac{\exp(-\eta \|q - q_{near}\|) \1[\world(q_{near}) = 0] + 1}{\exp(-\eta \|q - q_{near}\|) + 2}.
\end{equation*}
To estimate the posterior probability that an edge is collision-free, we take the minimum collision-free probability of discretized points along the edge.

Alternatively, if we know that worlds are uniformly drawn from a finite set of possible worlds, we can precompute the collision statuses for every edge in the graph against every world $\world$.
Then, the posterior is simply uniform over the remaining set of worlds that are consistent with $\history$.

This finite set posterior is one example of how planning algorithms may be able to leverage the structure existing in everyday environments.
The configuration space nearest-neighbor posterior only assumes that nearby configurations will have similar labels, which is a more broadly applicable (but less informative) structure.
This makes it well-suited for novel environments where the posterior does not have problem examples to learn from.

%% file: inputs/results.tex
\begin{figure*}[!t]
\centering

\includegraphics[height=0.19\linewidth]{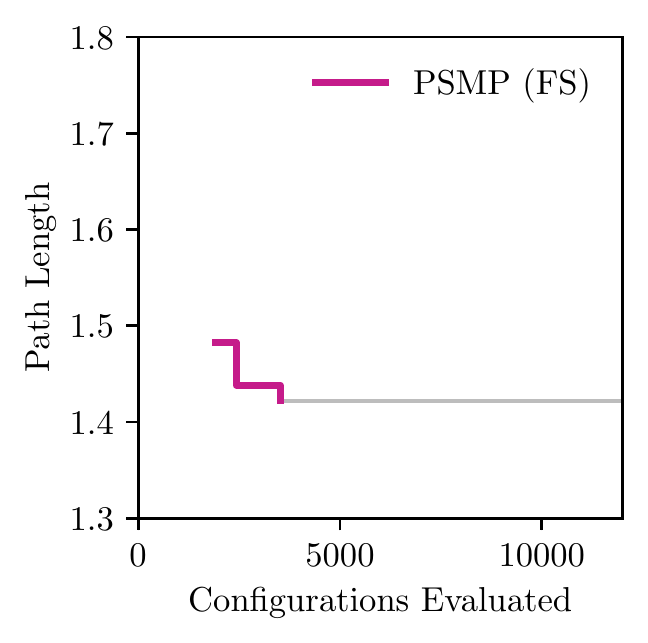}
\includegraphics[height=0.19\linewidth]{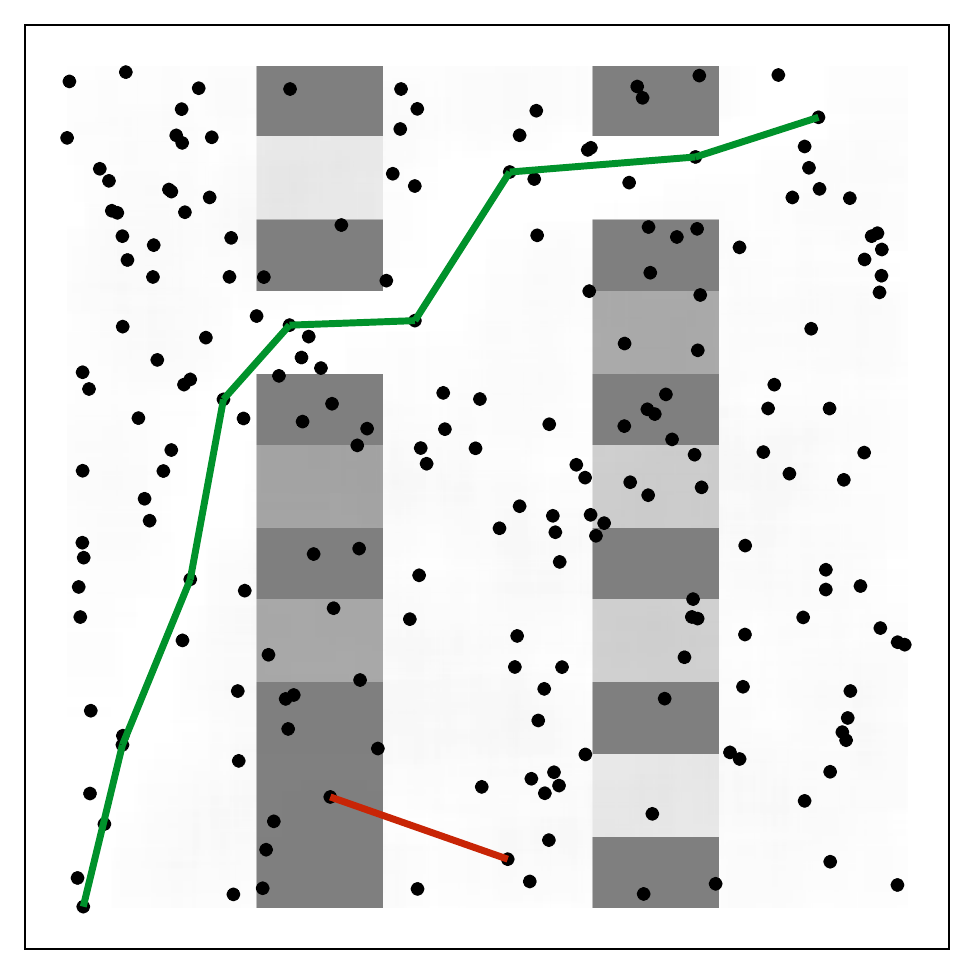}
\includegraphics[height=0.19\linewidth]{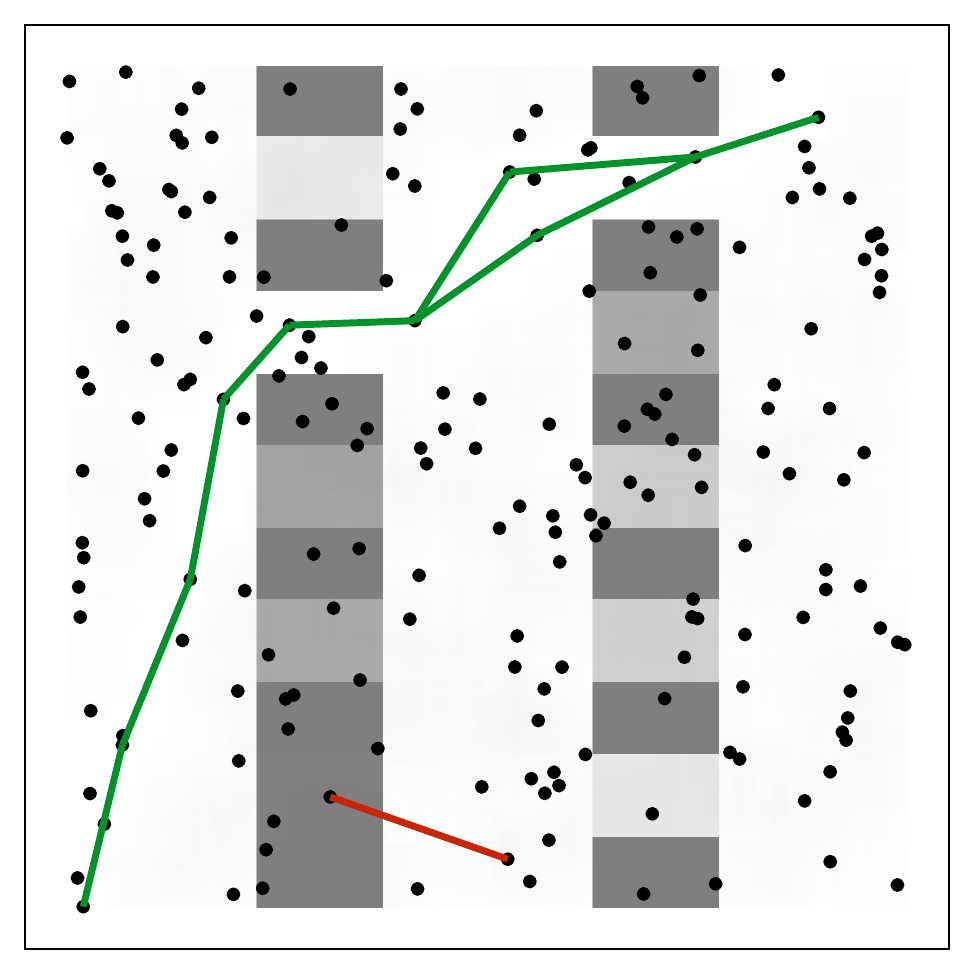}
\includegraphics[height=0.19\linewidth]{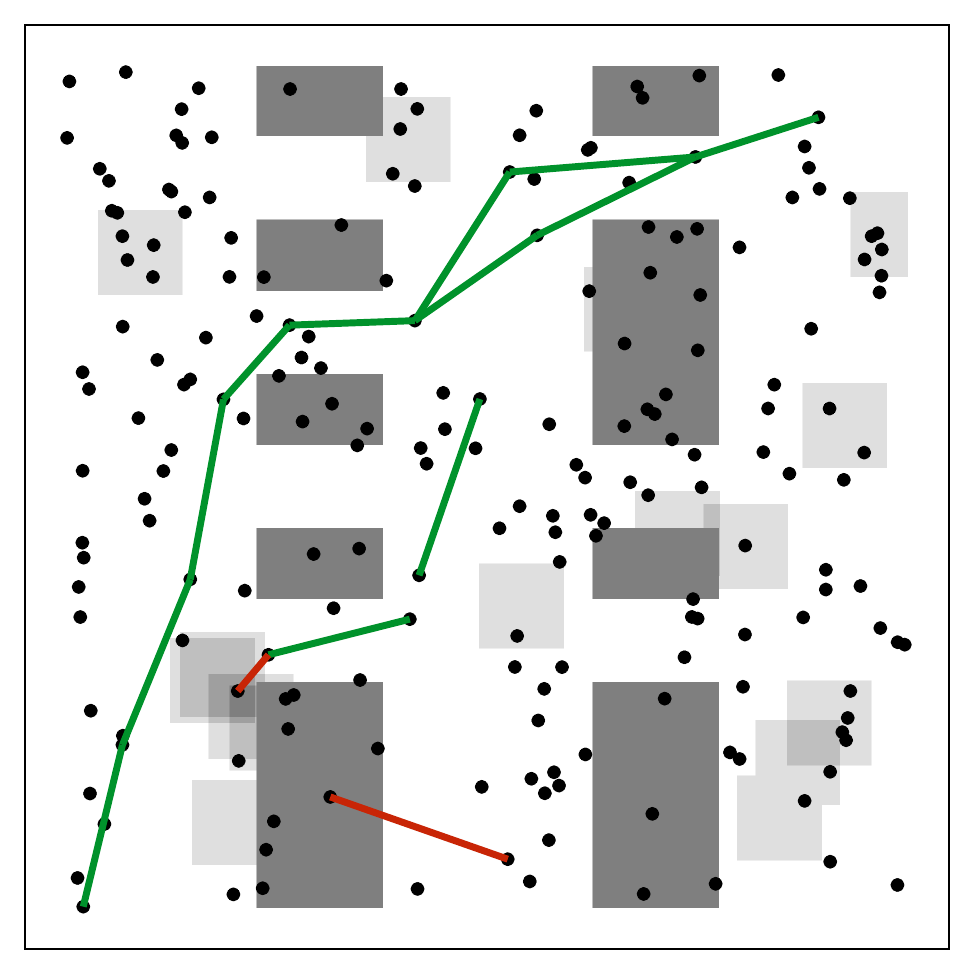}
\includegraphics[height=0.19\linewidth]{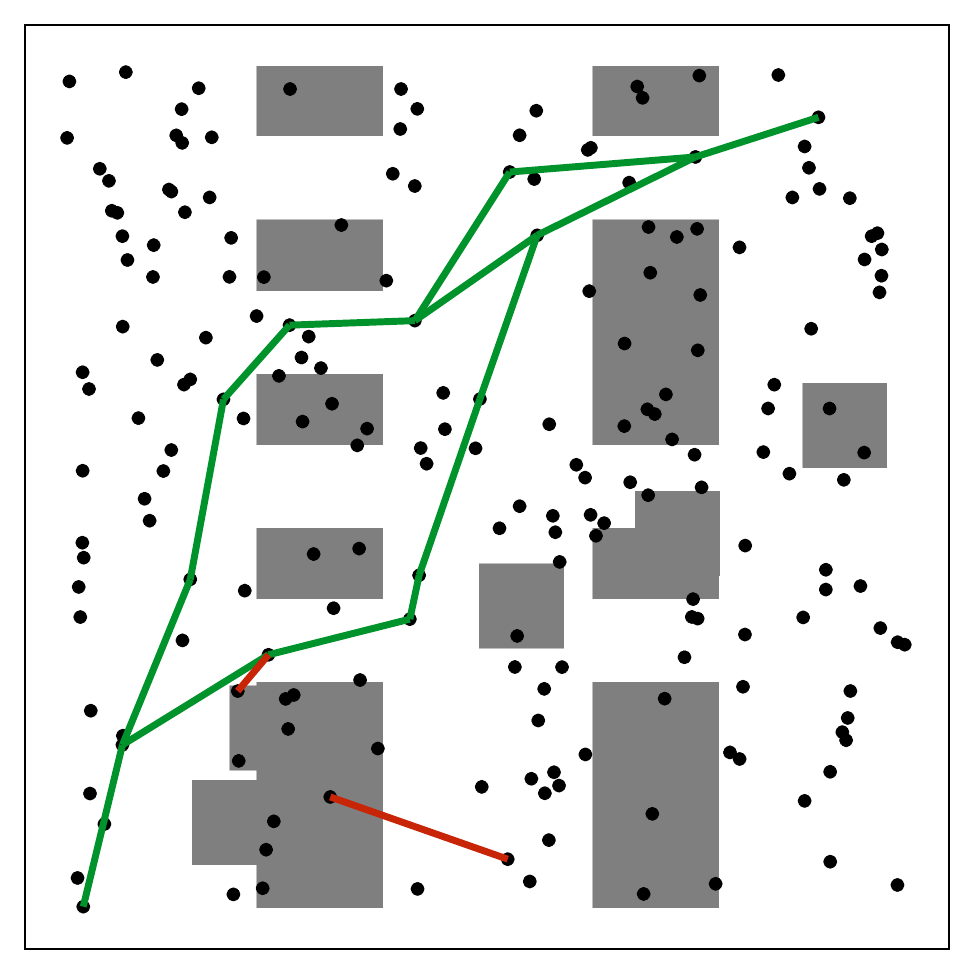}

\includegraphics[height=0.19\linewidth]{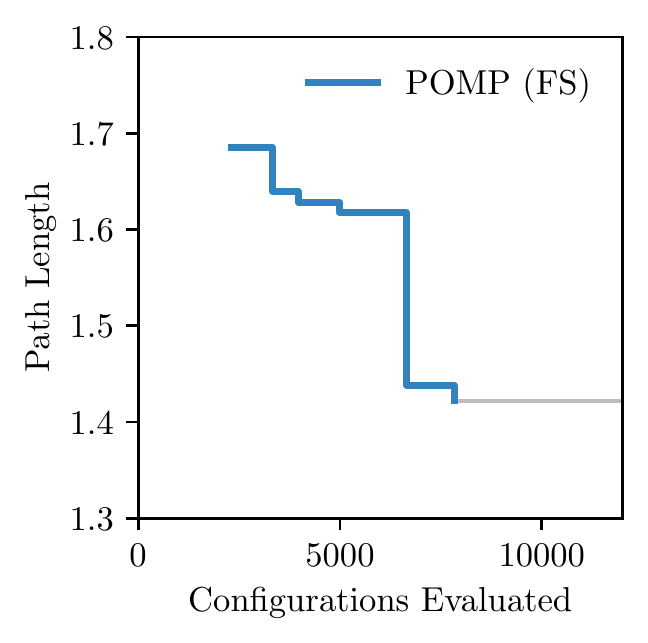}
\includegraphics[height=0.19\linewidth]{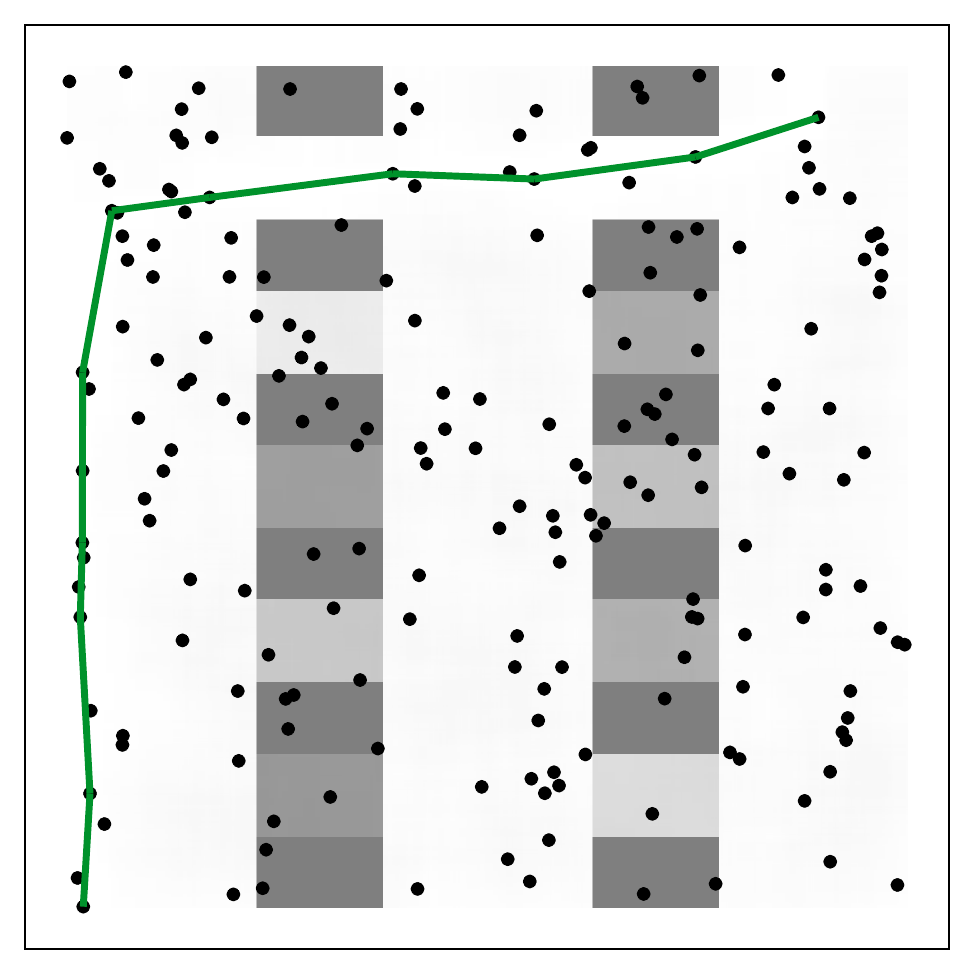}
\includegraphics[height=0.19\linewidth]{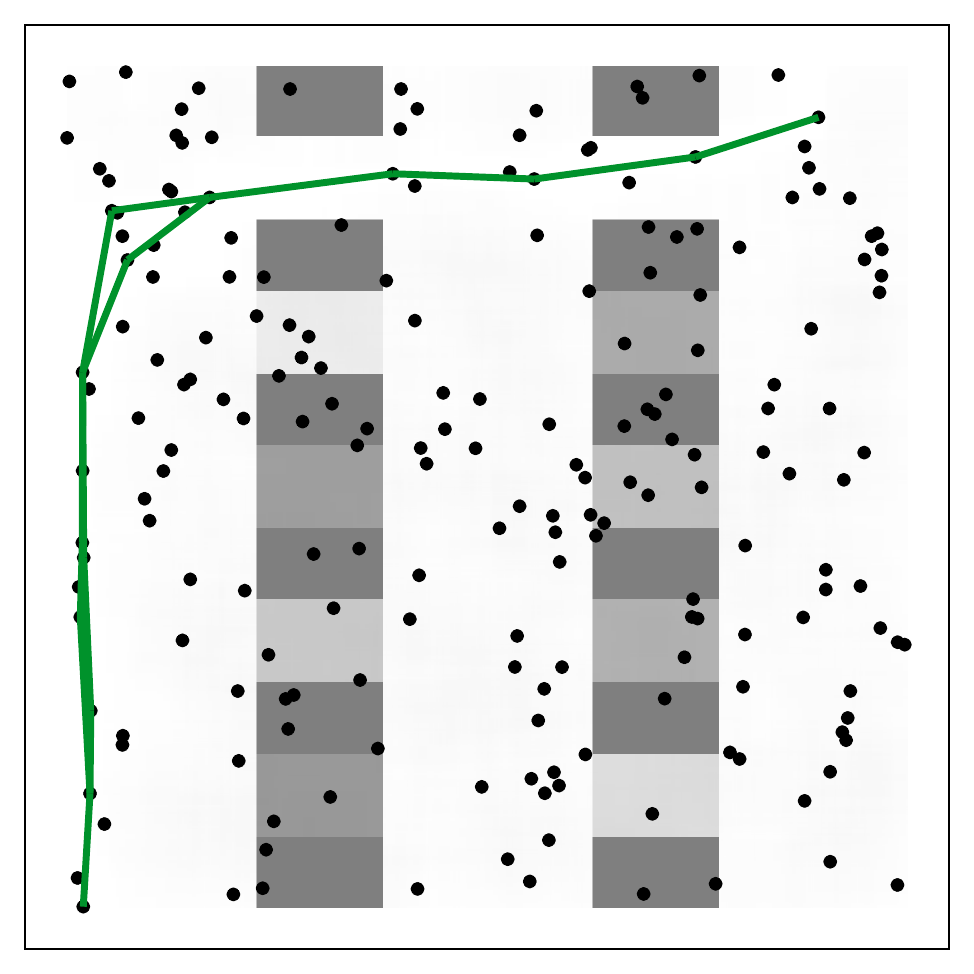}
\includegraphics[height=0.19\linewidth]{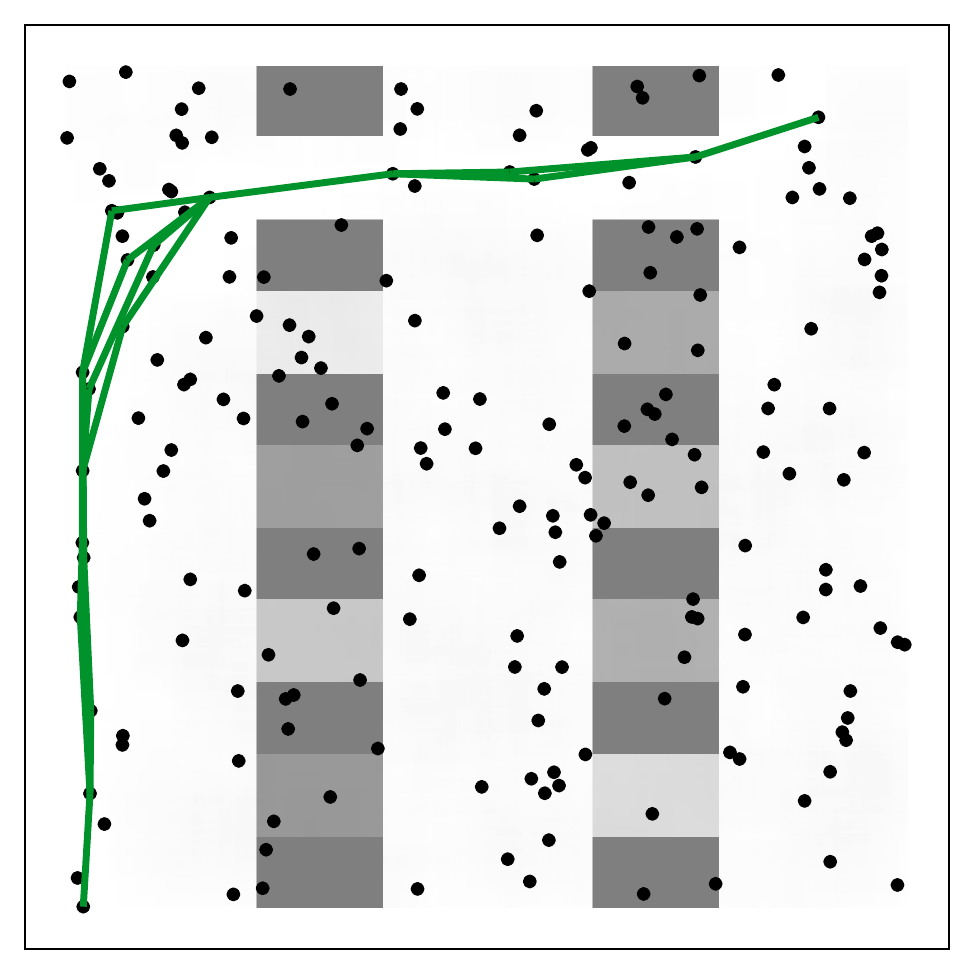}
\includegraphics[height=0.19\linewidth]{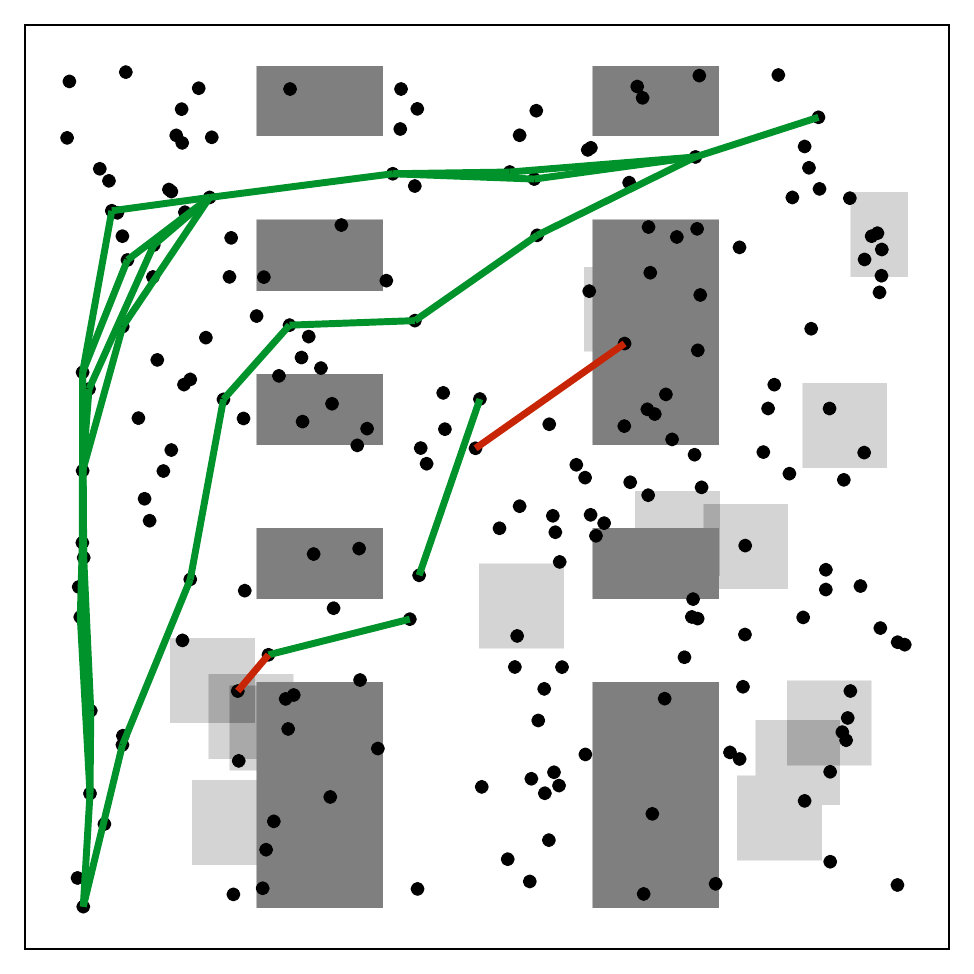}

\includegraphics[height=0.19\linewidth]{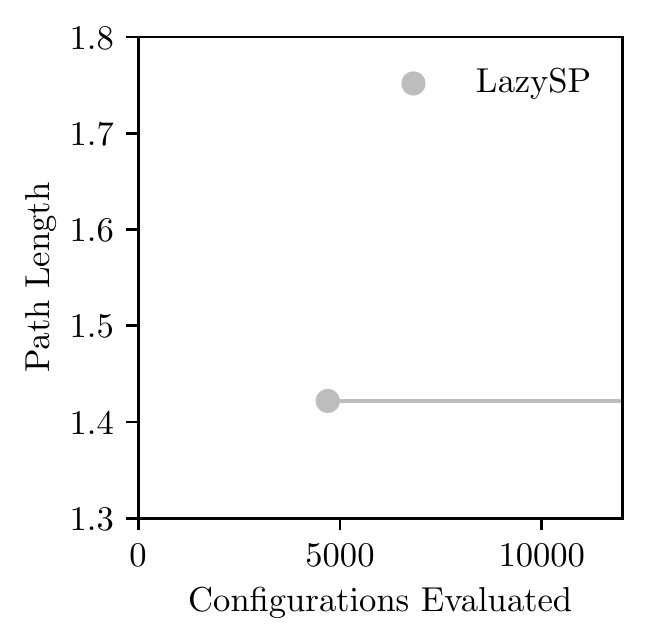}
\includegraphics[height=0.19\linewidth]{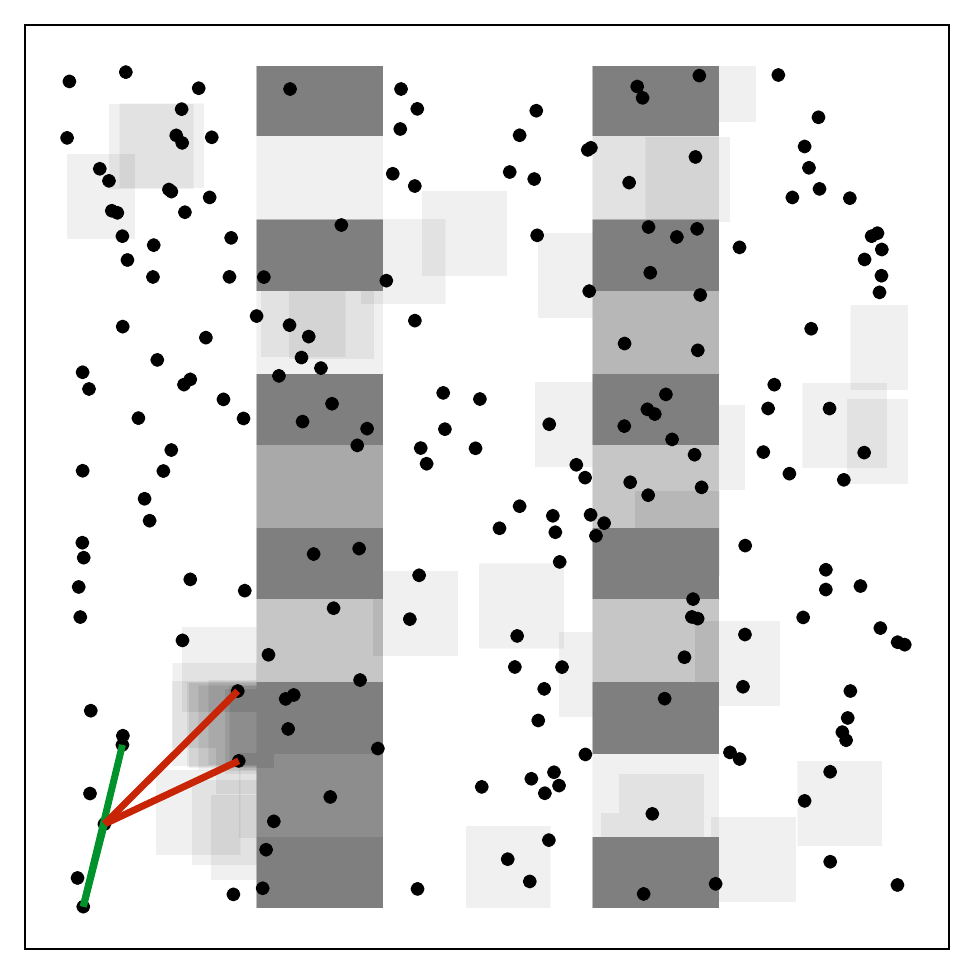}
\includegraphics[height=0.19\linewidth]{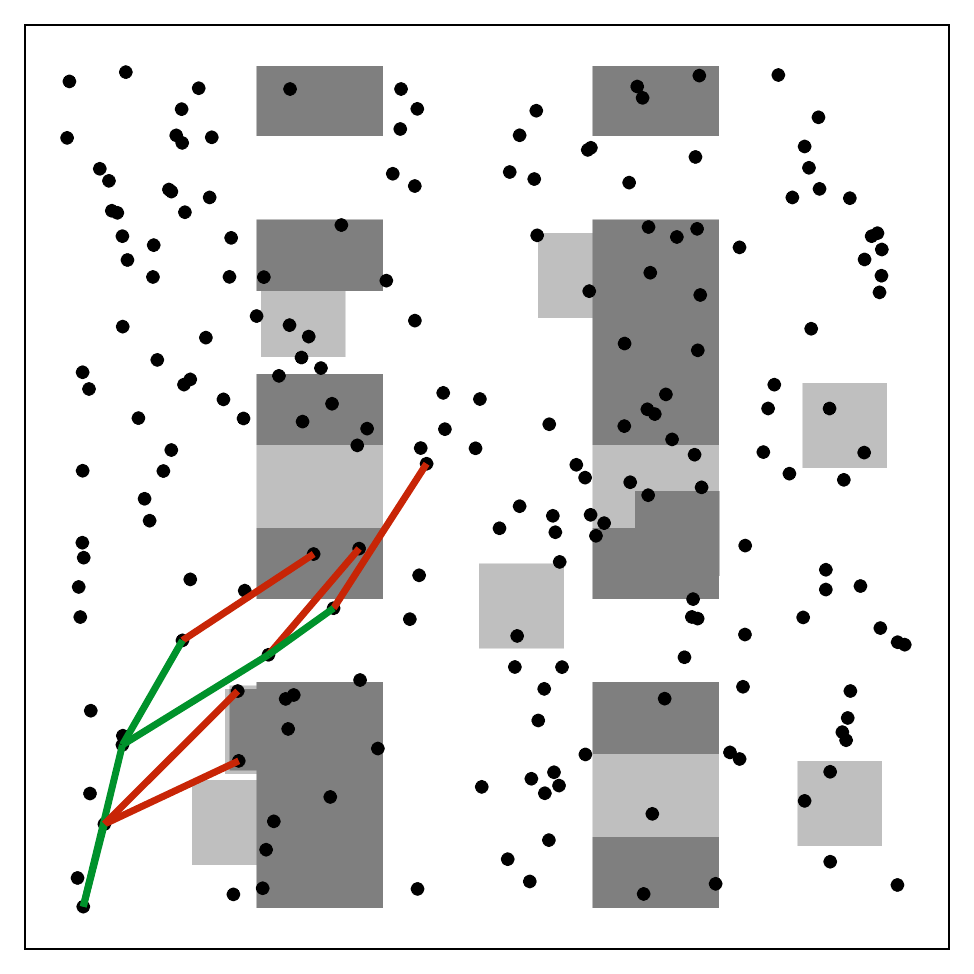}
\includegraphics[height=0.19\linewidth]{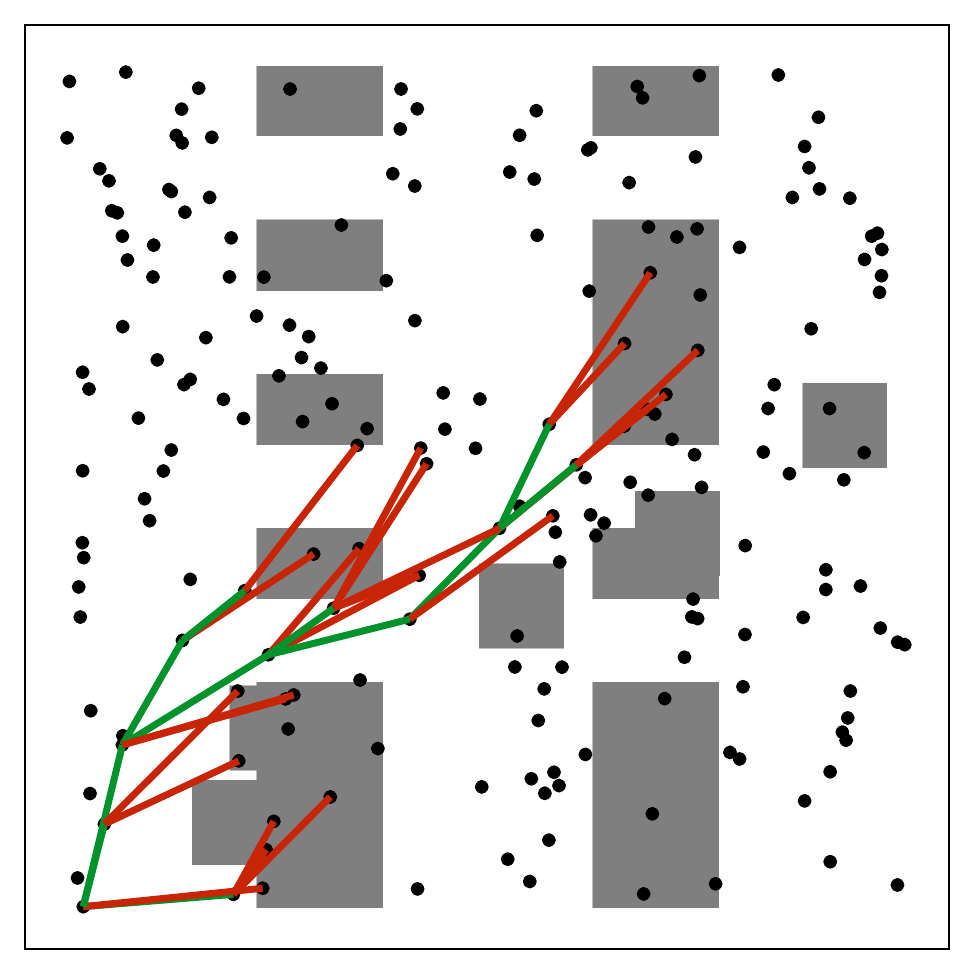}
\includegraphics[height=0.19\linewidth]{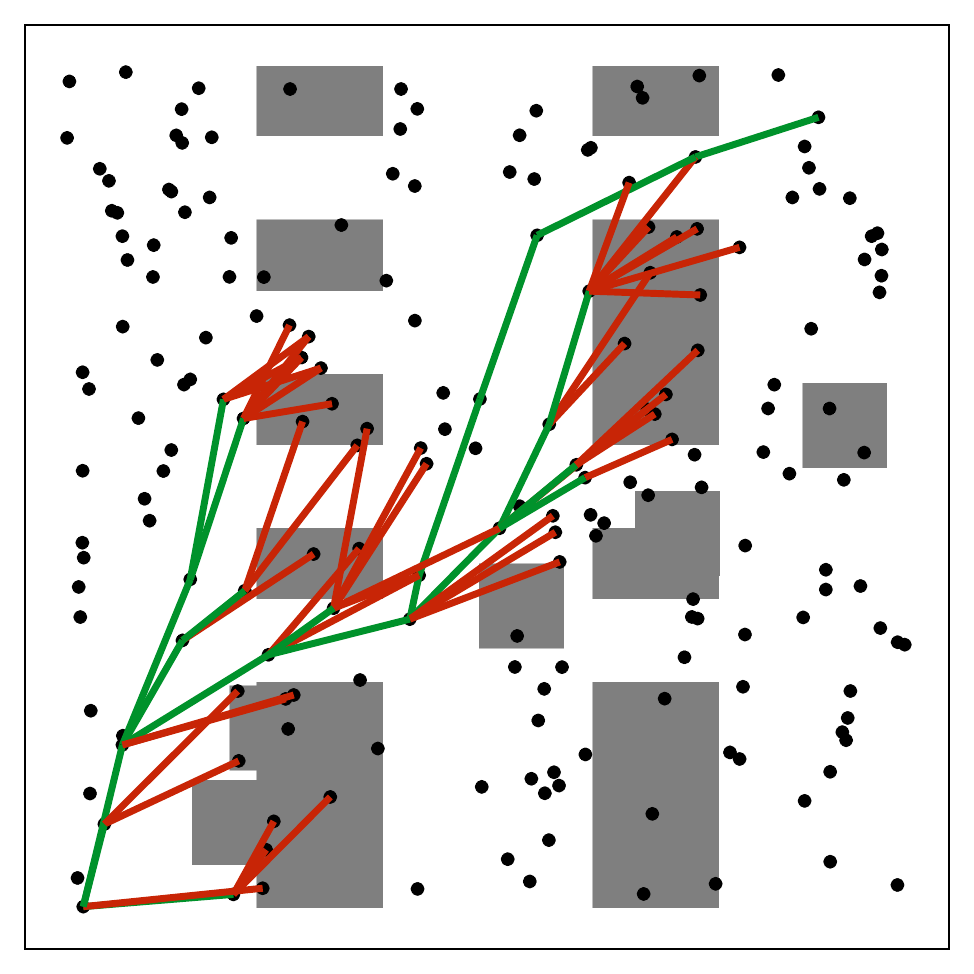}

\caption{
  (Left) Anytime performance of \algName (top), POMP (middle), and \lazysp (bottom).
  The gray line corresponds to the environment's shortest path.
  (Right) Snapshots of each algorithm's progress, with the finite set posterior used by \algName and POMP.
  Regions with higher probability of collision are colored with darker shades of gray.
  Evaluated edges are either found to be in collision (red) or collision-free (green).
}
\label{fig:filmstrip}
\vspace{-1\baselineskip}
\end{figure*}

\section{Experiments}
\label{sec:results}

We evaluate the anytime performance of \algName on 2-DOF~\cite{choudhury2017active} and 7-DOF motion planning datasets (\fref{fig:dataset}). 
The 7-DOF manipulator dataset was generated by randomly perturbing objects from an initial cluttered environment~\cite{qureshi2019mpnet}.
For 2-DOF problems, we compare \algName and POMP, with both the nearest neighbor-based (NN) and finite set (FS) posterior variants from \sref{subsec:method-posterior}.
We have highlighted their performance on two datasets here, and refer the reader to \extappref{sec:complete-results} for results on five more.
For 7-DOF problems, we only consider the finite set posterior for both \algName and POMP.
In this domain, we additionally combine RRTConnect with path shortening, a commonly-used heuristic for refining an initial feasible path~\cite{kuffner2000rrt}.
Because collision checking dominates planning time, we report the number of configurations checked by each algorithm.
We refer the reader to \extappref{sec:experimental-details} for further experimental details.

We visualize sample runs by \algName, POMP, and \lazysp in \fref{fig:filmstrip}.
Note that since \lazysp is not an anytime algorithm, it only produces one solution.
Using the same finite set posterior as POMP, \algName finds a shorter feasible path with fewer collision checks.
Furthermore, it returns the shortest path faster than the uninformed \lazysp baseline.
POMP carefully attempts to avoid edges that may be in collision;
as a result, refining the initial feasible solution can take a substantial amount of time.

The performance of these algorithms on the remainder of the test set (200 environments) is summarized in the first row of \fref{fig:plots}.
The anytime performance of the nearest neighbor-based (NN) and finite set (FS) posteriors have been separated for clarity.
\algName with the FS posterior has a marked improvement over the NN variant:
it finds initial feasible solutions faster than all other algorithms and quickly refines them.
However, averaging the worlds in the feasible set that are consistent with the evaluation history---as POMP does---can lead to over-exploration of impossible scenarios and degraded performance.

In the maze environments (\fref{fig:plots}, middle), \algName with the FS posterior continues to outperform other algorithms.
However, the evaluation history quickly narrows the posterior to a single feasible world, so POMP enjoys similar performance to \algName.
In these settings where there are very few feasible paths, \algName with the NN posterior does not perform as well.
We believe that this is because an anytime objective implicitly assumes the existence of multiple feasible paths.
Although posterior sampling explores options quickly, exploration may be unnecessary if this assumption is violated.
In such scenarios, a posterior that captures more global correlations in the environment may be needed for improved performance relative to an algorithm like \lazysp.

On the 7-DOF manipulator environments, we chose to compare with heuristically shortening an initial feasible RRT path (RRT+PS) rather than RRT*~\cite{karaman2011sampling};
although the latter guarantees asymptotic optimality, it empirically takes much longer to find initial solutions (\extappref{sec:complete-results}).
While RRT+PS finds a feasible path faster than \lazysp, it needs more collision checks and emits longer paths than \algName or POMP with the FS posterior~(\fref{fig:plots}, bottom).

\begin{figure*}
\centering

\includegraphics[height=0.3\linewidth]{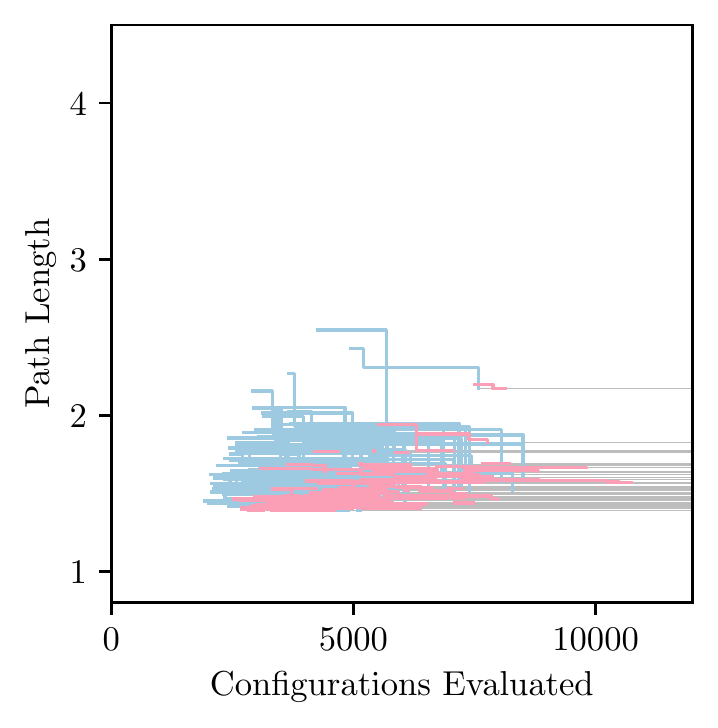}
\hspace{1em}
\includegraphics[height=0.3\linewidth]{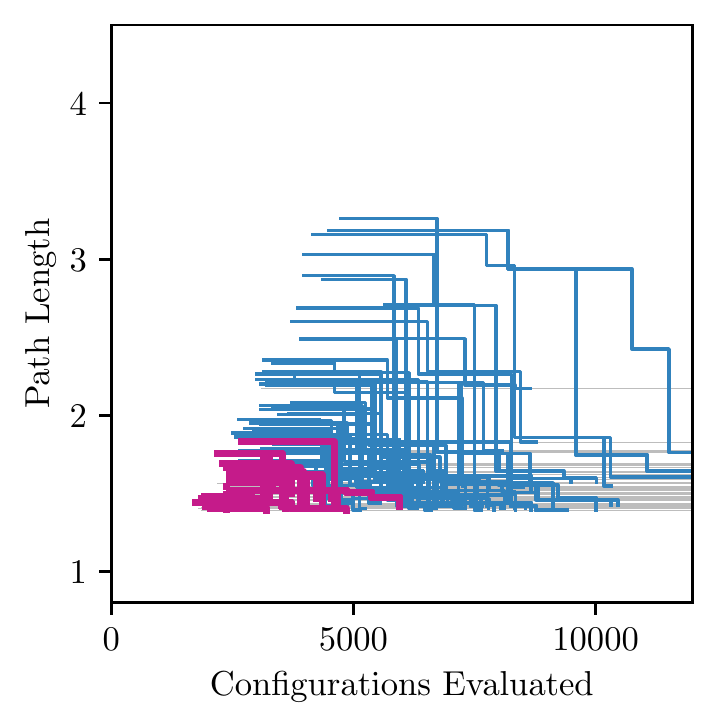}
\hfill
\includegraphics[height=0.3\linewidth]{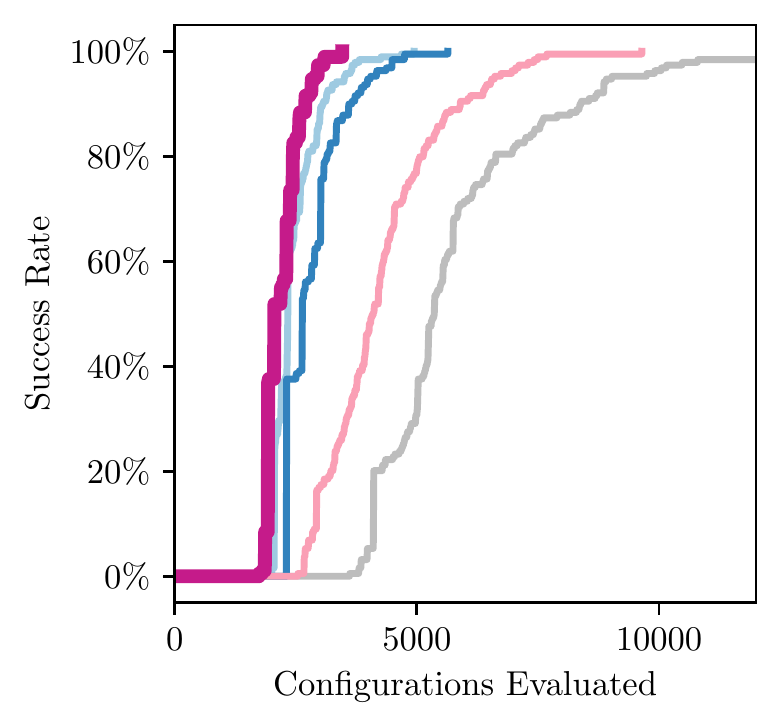}

\includegraphics[height=0.3\linewidth]{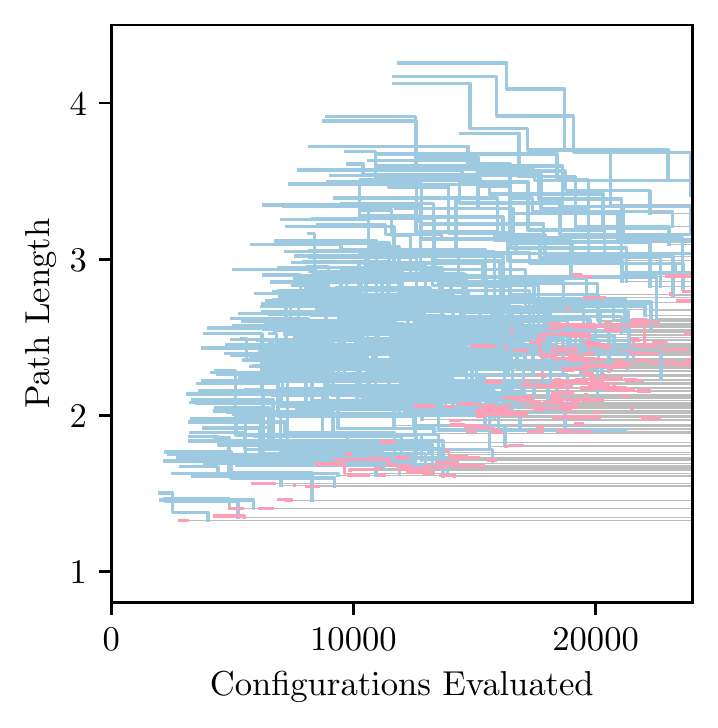}
\hspace{1em}
\includegraphics[height=0.3\linewidth]{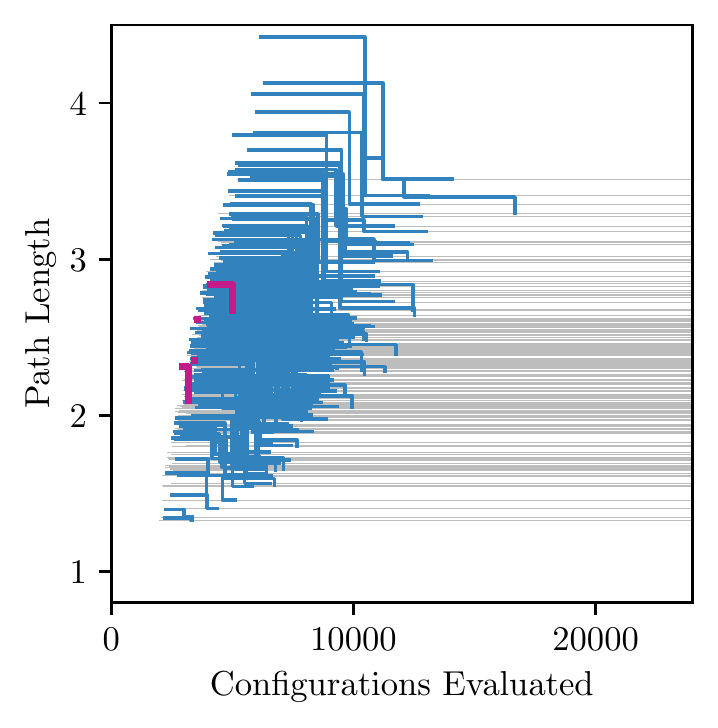}
\hfill
\includegraphics[height=0.3\linewidth]{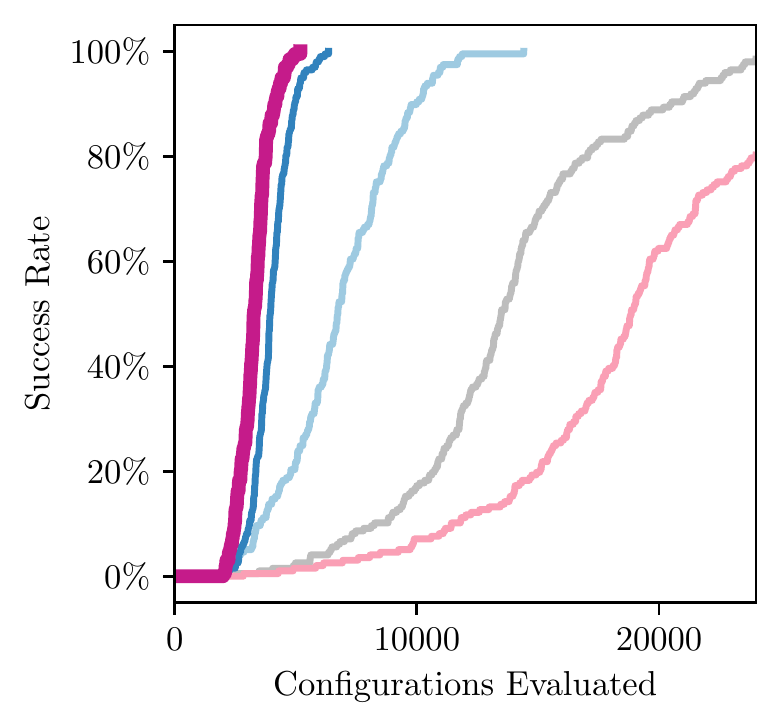}

\includegraphics[height=0.3\linewidth]{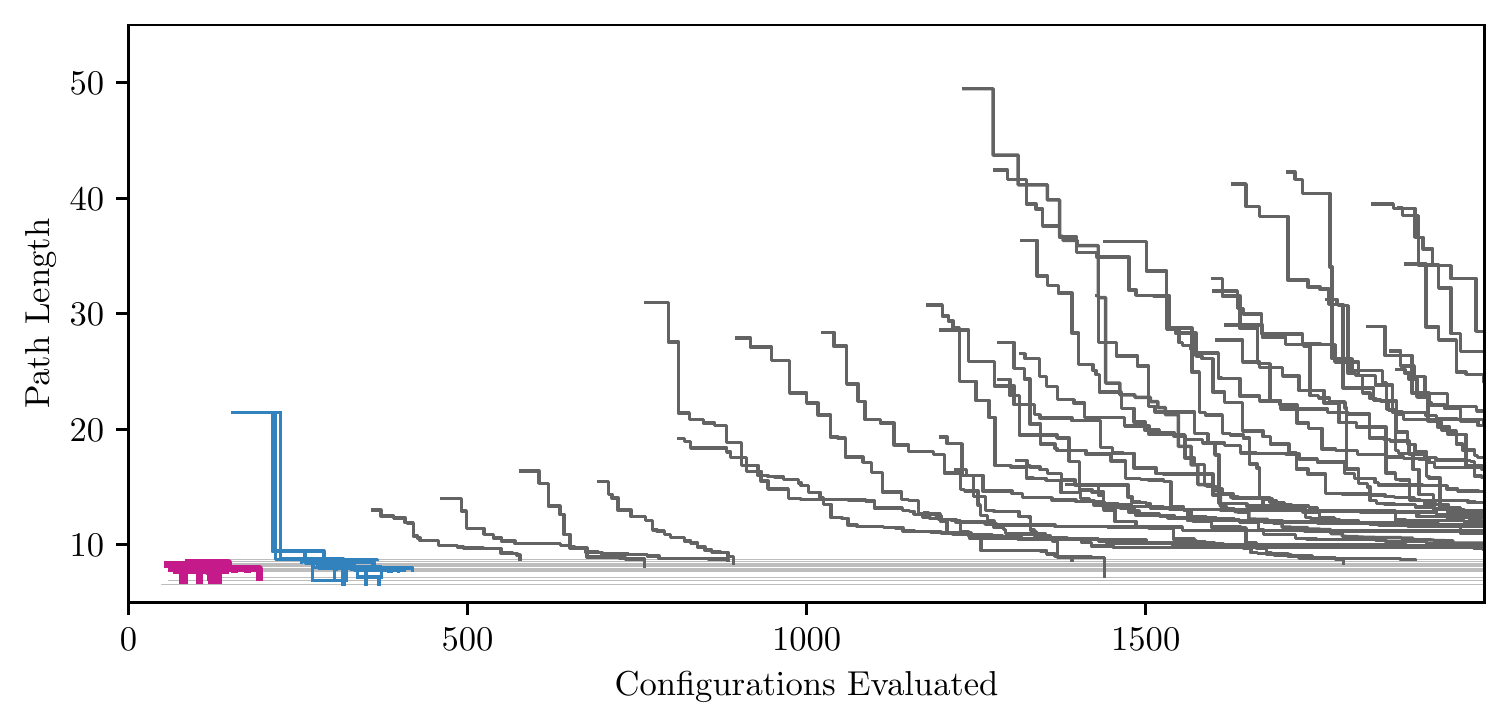}
\hfill
\includegraphics[height=0.3\linewidth]{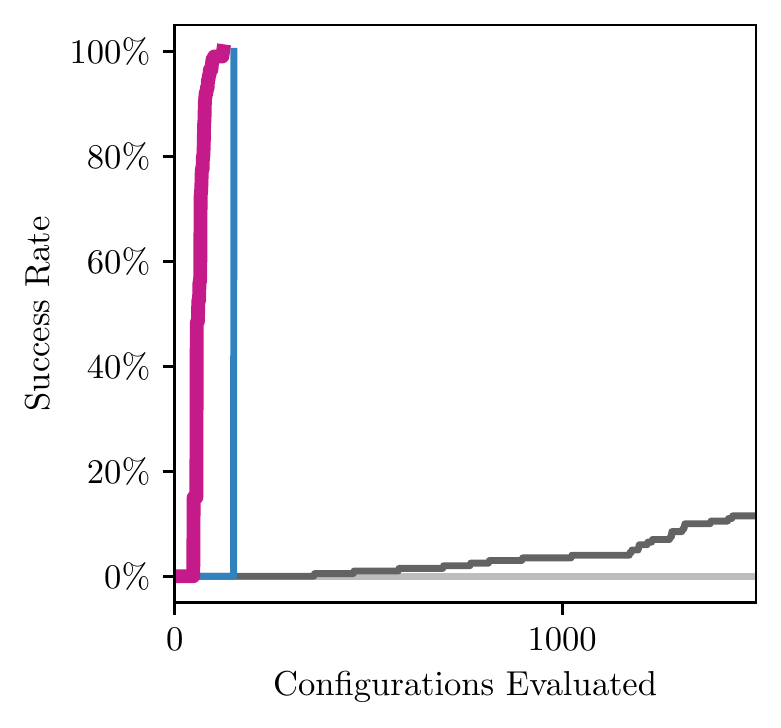}

\includegraphics[width=0.5\linewidth]{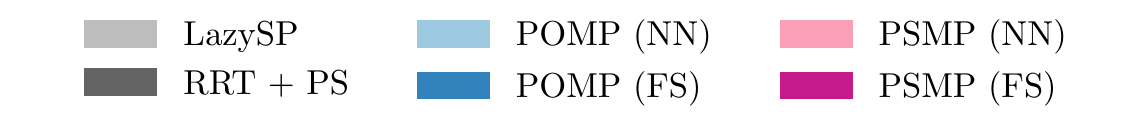}

\caption{
  (Left) Length of the best feasible path discovered by each anytime algorithm over time.
  (Right) Collision checking budget versus the percentage of planning problems where that budget is sufficient to discover a feasible path.
  (Top and Middle) 2-DOF environments corresponding to (\fref{fig:dataset}, Left and Right).
  (Bottom) 7-DOF robot manipulator environment, corresponding to (\fref{fig:dataset}, Bottom).
}
\label{fig:plots}
\vspace{-0.5\baselineskip}
\end{figure*}


%% file: inputs/discussion.tex

\section{Discussion and Future Work}
\label{sec:discussion}

Anytime algorithms should rapidly discover progressively shorter paths.
We have formalized this intuitive objective as minimizing Bayesian regret.
Sublinear Bayesian regret---which \algName achieves---implies asymptotic optimality, while demanding good intermediate performance.
We hope that drawing this connection between anytime motion planning and Bayesian reinforcement learning will open the door to further regret analysis of anytime algorithms.

In this work, we have focused on the problem of anytime search on a fixed graph.
Many existing anytime algorithms take an incremental densification approach, requiring new samples from the configuration space for continued improvement.
Regret analysis for these continuous-space problems is an open challenge.

Empirically, \algName has strong performance and improves further when more knowledge is incorporated into the posterior structure.
Using data from previous planning environments to learn the underlying structure of collision posteriors, via e.g. unsupervised generative models, will enable \algName to quickly solve new instances of those problems.

%% file: inputs/arxiv-link.tex
\section{Appendices}

This paper is available with appendices at \url{https://arxiv.org/abs/2002.11853}.

%% file: inputs/appendix.tex

\section*{Appendix}
\label{sec:appendix}

\renewcommand\thesubsection{\Alph{subsection}}

\subsection{Proof of \thmref{thm:regret}}
\label{sec:regret-proof}

\begin{proof}
We begin by noting that regret is measured w.r.t the best discovered policy $\hat{\policy}_k$ which is history-dependent, i.e., dependent on $\seq{\policy}{k}$.
Hence, we upper bound it with an alternative version of regret w.r.t the executed policy $\policy_k$. 
\begin{equation*}
\begin{aligned}
  \Delta_k
  &= V^{\mdp^*}_{\policy^*}(\state_1) - V^{\mdp^*}_{\hat{\policy}_k}(\state_1) \\
  &\leq V^{\mdp^*}_{\policy^*}(\state_1) - V^{\mdp^*}_{\policy_k}(\state_1)
  \leq \bar{\Delta}_k
\end{aligned}
\end{equation*}
Posterior sampling leverages the fact that $\mdp^*$ and $\mdp_k$ are identically distributed.
One hurdle in the analysis is that the optimal policy $\policy^*$ is not directly observed.
Hence, we introduce yet another notion of regret which does not depend on $\policy^*$.
\begin{equation*}
  \tilde{\Delta}_k = V^{\mdp_k}_{\policy_k}(\state_1) - V^{\mdp^*}_{\policy_k}(\state_1) \\
\end{equation*}
which is the difference in expected value of the policy $\policy_k$ under the sampled MDP $\mdp_k$ and the true MDP $\mdp^*$ which is observed.
We apply Theorem 2 from \cite{osband2013more} to show that the two regrets are equal in expectation
\begin{equation*}
 \expect{}{\sum_{k=1}^m \bar{\Delta}_k} = \expect{}{\sum_{k=1}^m \tilde{\Delta}_k} 
\end{equation*}
with high probability.

We will now bound $\expect{}{\sum_{k=1}^m \tilde{\Delta}_k}$.
Unlike the analysis in \cite{osband2013more} for the stochastic case, the deterministic regret is much easier to bound.
It amounts to the difference in rewards observed in $\mdp^*$ versus $\mdp_k$.
We will bound this by arguing that $\mdp^*$ concentrates around $\mdp_k$ using the notion of confidence sets as in \citep{jaksch2010near}.
Let $t_k$ be the time at the beginning of the $k^{th}$ episode.
Let $\hat{\reward}(\state, \action)$ be the empirical average reward and $ N_{t_k}(\state, \action)$ be the number of times $(\state, \action)$ was queried.
We define the following confidence set for episode $k$
\begin{equation*}
  \mdpSet_k  = \{ \mdp : |\hat{\reward}(\state, \action) - \reward^\mdp(\state,\action)| \leq \beta_k(\state,\action) \quad \forall (\state, \action) \}
\end{equation*}
where $\beta_k(\state,\action) = \sqrt{ \frac{7 \log (2 \stateSet \actionSet m t_k)}{\max\{1, N_{t_k}(\state, \action) \}} }$ is chosen to ensure both $\mdp_k$ and $\mdp^*$ belong to $\mdpSet_k$ with high probability as specified in \citep{jaksch2010near}.
We bound regret as follows:
\begin{equation*}
\begin{aligned}
& \expect{}{\sum_{i=1}^m \tilde{\Delta}_k}\\
& \leq \expect{}{\sum_{k=1}^m \tilde{\Delta}_k \Ind(\mdp_k, \mdp^* \in \mdpSet_k)} + 2 \tau \sum_{k=1}^m P(\mdp^* \notin \mdpSet_k) \\
& \leq \expect{}{\sum_{k=1}^m \expect{}{\tilde{\Delta}_k | \mdp^*, \mdp_k} \Ind(\mdp_k, \mdp^* \in \mdpSet_k)} + 2 \tau \\
& \leq \expect{}{\sum_{k=1}^m \sum_{i=1}^\tau \min \{ \beta_k (\state_{t_k + i}, \action_{t_k + i}), 1\}} + 2 \tau \\
& \leq \min\{ \horizon \sum_{k=1}^m \sum_{i=1}^\tau \min \{ \beta_k (\state_{t_k + i}, \action_{t_k + i}), 1\}, T \}
\end{aligned}
\end{equation*}
where the second inequality follows from the fact that Lemma 17 of \citep{jaksch2010near} shows $P(\mdp^* \notin \mdpSet_k) \leq \frac{1}{m}$. The final inequality follows from the fact that worst case regret is bounded by $T$.

We now bound
\begin{equation*}
  \textstyle
  \min\{ \horizon \sum_{k=1}^m \sum_{i=1}^\tau \min \{ \beta_k (\state_{t_k + i}, \action_{t_k + i}), 1\}, T \}
\end{equation*}
First note that
\begin{equation*}
\begin{aligned}
  &\sum_{k=1}^m \sum_{i=1}^\tau \beta_k(s, a)\\
  &\leq \sum_{k=1}^m \sum_{i=1}^\tau \Ind(N_{t_k} \leq \tau) + \sum_{k=1}^m \sum_{i=1}^\tau \Ind(N_{t_k} > \tau) \beta_k(s, a)
\end{aligned}
\end{equation*}
The first term is shown to be bounded.
\begin{equation*}
  \textstyle
  \sum_{k=1}^m \sum_{i=1}^\tau \Ind(N_{t_k} \leq \tau) \leq 2 \tau \stateSet \actionSet
\end{equation*}
The second term utilizes the following bound
\begin{equation*}
  \textstyle
  \sum_{k=1}^m \sum_{i=1}^\tau \sqrt{ \frac{ \Ind(N_{t_k} > \tau) }{\max\{1, N_{t_k}(\state, \action) \}} } \leq \sqrt{2 \stateSet \actionSet T}
\end{equation*}
We can now bound
\begin{equation*}
\begin{aligned}
&\min\{ \horizon \sum_{k=1}^m \sum_{i=1}^\tau \min \{ \beta_k (\state_{t_k + i}, \action_{t_k + i}), 1\}, T \} \\
&\leq \min\{ 2\tau^2 \stateSet \actionSet + \tau\sqrt{14 \stateSet \actionSet T \log ( \stateSet \actionSet T )}, T\} \\
&\leq \tau \sqrt{ 16 \stateSet \actionSet T \log( \stateSet \actionSet T) }
\end{aligned}
\end{equation*}

Hence the Bayesian regret is bounded by $O(\tau \sqrt{ \stateSet \actionSet T \log( \stateSet \actionSet T) })$. 
\end{proof}

\subsection{Experimental Details}
\label{sec:experimental-details}

\subsubsection{Nearest Neighbor-Based (NN) Posterior}

We assume a uniform $\text{Beta}(1, 1)$ prior on configuration space collision probability.
The status of the nearest neighbor to query point $q$ counts as a partial success or failure with weight $\exp(-\eta \|q - q_{near}\|)$.
For the 2-DOF planning environments, we use $\eta = 10^3$.
To estimate the posterior probability that an edge is collision-free, we take the minimum collision-free probability of 5 discretized points along the edge.

\subsubsection{Collision-Checking}

Each edge is collision-checked up to a fixed resolution via binary search.
For the 2-DOF planning environments, both dimensions range from 0 to 1 and edges are checked at a resolution of 0.001.
Edges are checked at a resolution of 0.2 for the 7-DOF manipulator planning environments.

\subsubsection{POMP}

In each iteration, $\alpha$ controls the trade-off between collision probability and edge weight.
The algorithm starts with $\alpha = 0$ and increases to $\alpha = 1$ as new feasible paths are discovered.
We use the same step size of 0.1 as the original paper.

\subsubsection{RRT+PS and RRT*}

We used the OMPL implementations of RRTConnect and RRT*.
On each of the 7-DOF environments, the algorithms search until they discover a path that is shorter than the shortest path in the graph (or until a maximum time limit has been exceeded).
RRT+PS was evaluated with a 5 second timeout, while RRT* needed an increased limit of 30 seconds to return feasible solutions for all environments.

Path shortcutting is implemented with the OMPL default parameters.
For both RRTConnect and RRT*, we set the range to infinity.
Intermediate states were added to the tree for RRTConnect.
For RRT*, we turned on lazy collision-checking and focused search (for pruning and informed sampling once a feasible path was discovered).
We used the default RRT* rewiring factor of 1.1.

\subsection{Complete Experiments}
\label{sec:complete-results}

For the 7-DOF manipulator experiments, we compared with the stronger baseline of RRT+PS rather than RRT*.
While the RRT* algorithm is an anytime algorithm with asymptotic optimality guarantees, it is much slower than RRT+PS in practice (\fref{fig:plots_7d_rrt}).

Fig.~\ref{fig:plots_2d_a} and \ref{fig:plots_2d_b} show results for the remaining five 2-DOF datasets not evaluated in \sref{sec:results}.
Example problems from each dataset are visualized in the left column.
The center two columns show the length of the best feasible path discovered by each anytime algorithm over time, with the nearest-neighbor posterior on the left and finite set posterior on the right.
In the last column, the collision checking budget is plotted versus the percentage of planning problems where that budget is sufficient to discover a feasible path.

\begin{figure*}
\centering

\hfill
\includegraphics[height=0.27\linewidth]{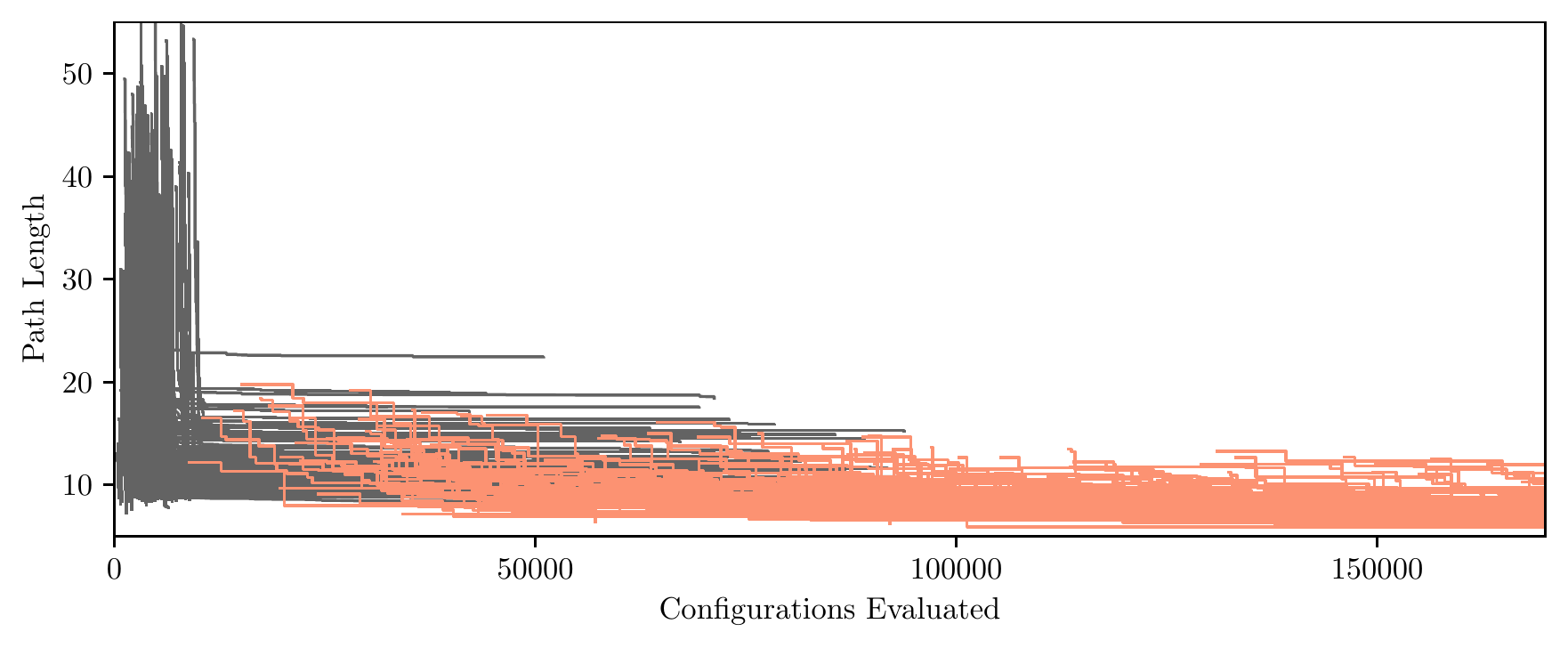}
\hfill
\includegraphics[height=0.27\linewidth]{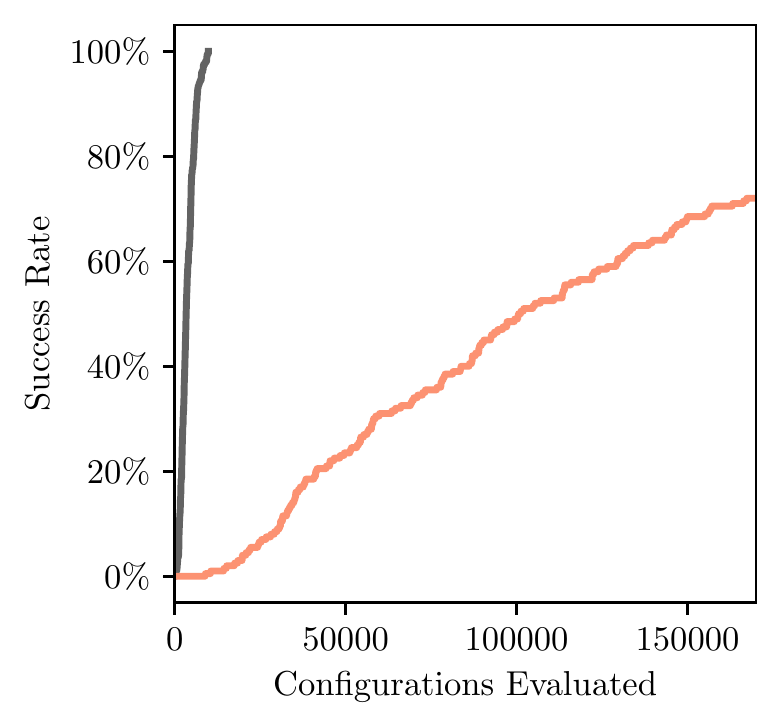}

\includegraphics[width=0.5\linewidth]{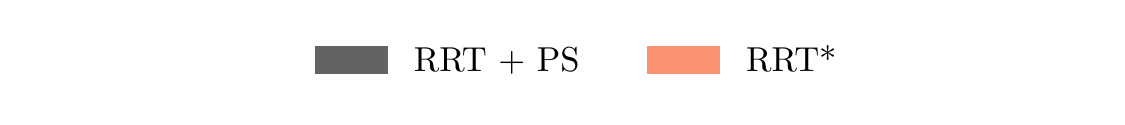}

\caption{
  Compared to RRT+PS, RRT* requires many more collision checks before it can emit any feasible solutions.
  The x-axis of this plot is about 100 times larger than in the bottom row of \fref{fig:plots}.
  Even with this significantly larger collision-checking budget, RRT* only finds a feasible path in 70\% of environments.
  However, the initial feasible paths RRT* discovers are much shorter than those from RRTConnect.
}
\label{fig:plots_7d_rrt}
\end{figure*}

\begin{figure*}
\centering

\begin{subfigure}{0.08\linewidth}
\frame{\includegraphics[height=\linewidth]{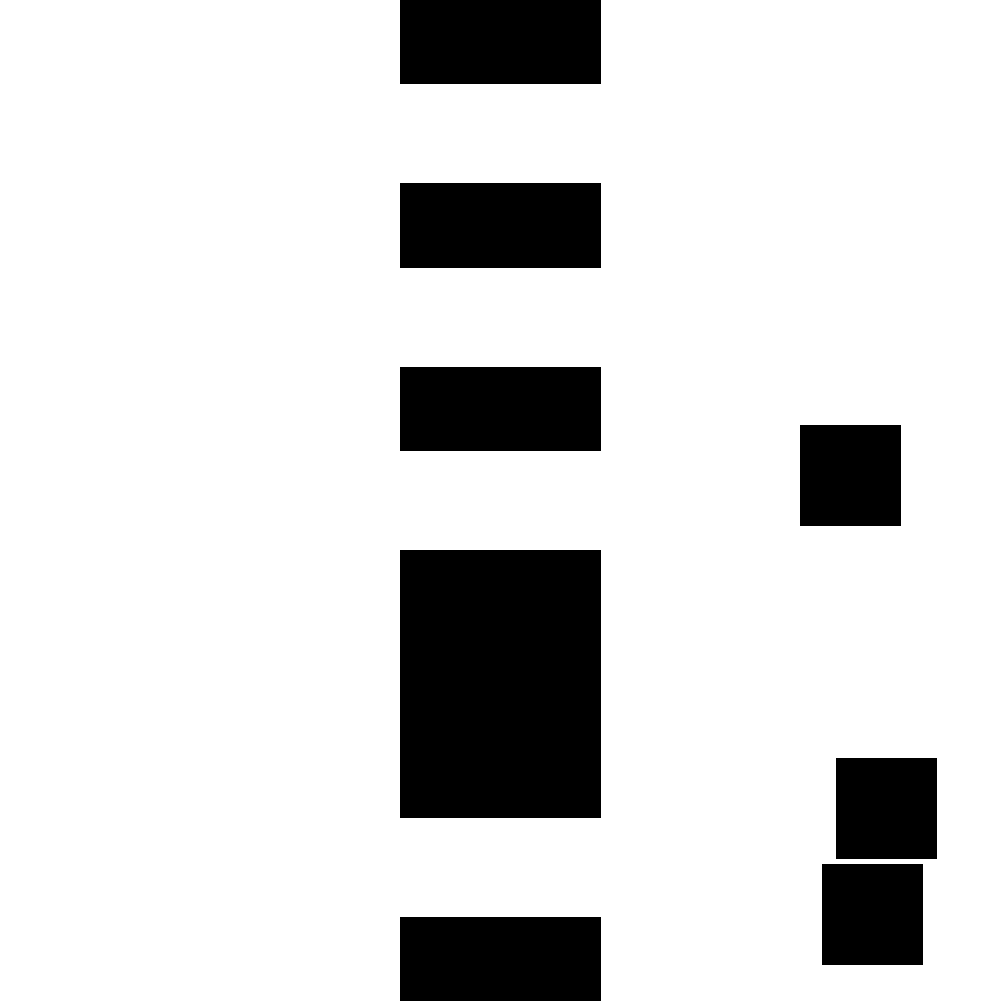}}

\frame{\includegraphics[height=\linewidth]{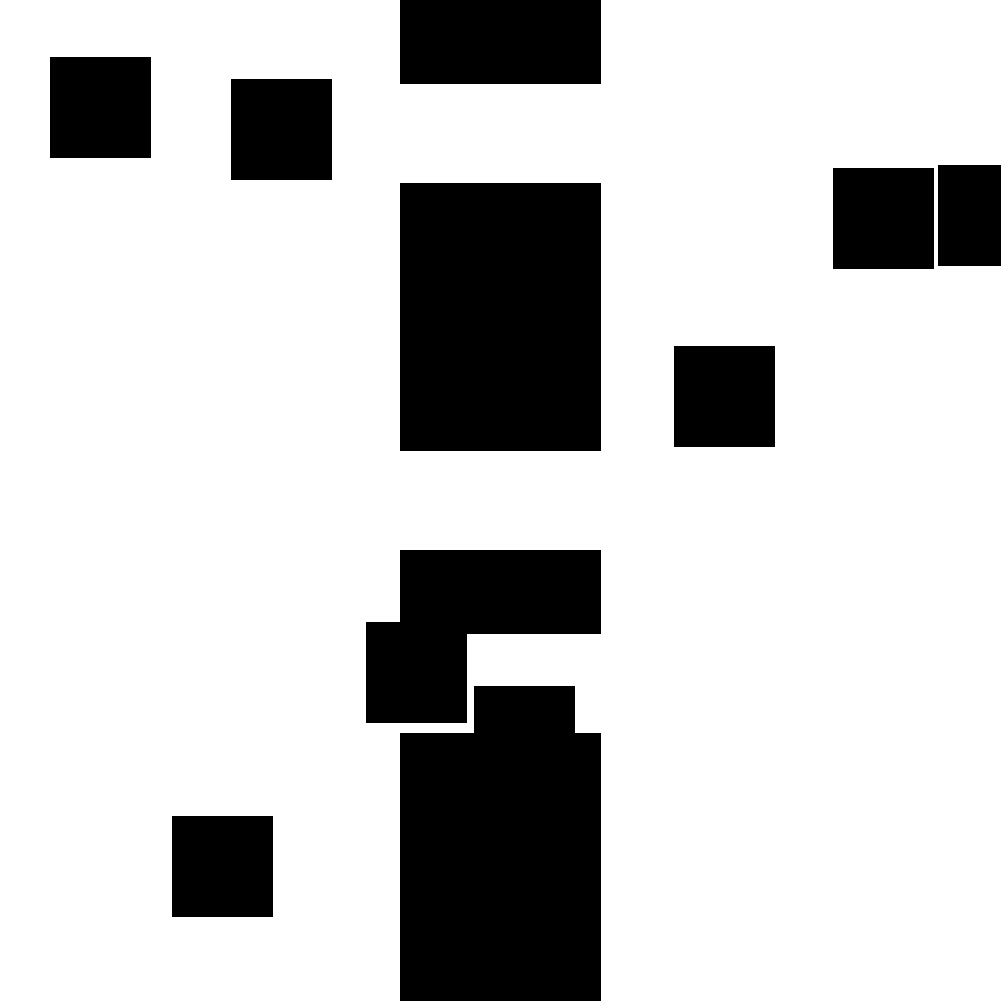}}

\frame{\includegraphics[height=\linewidth]{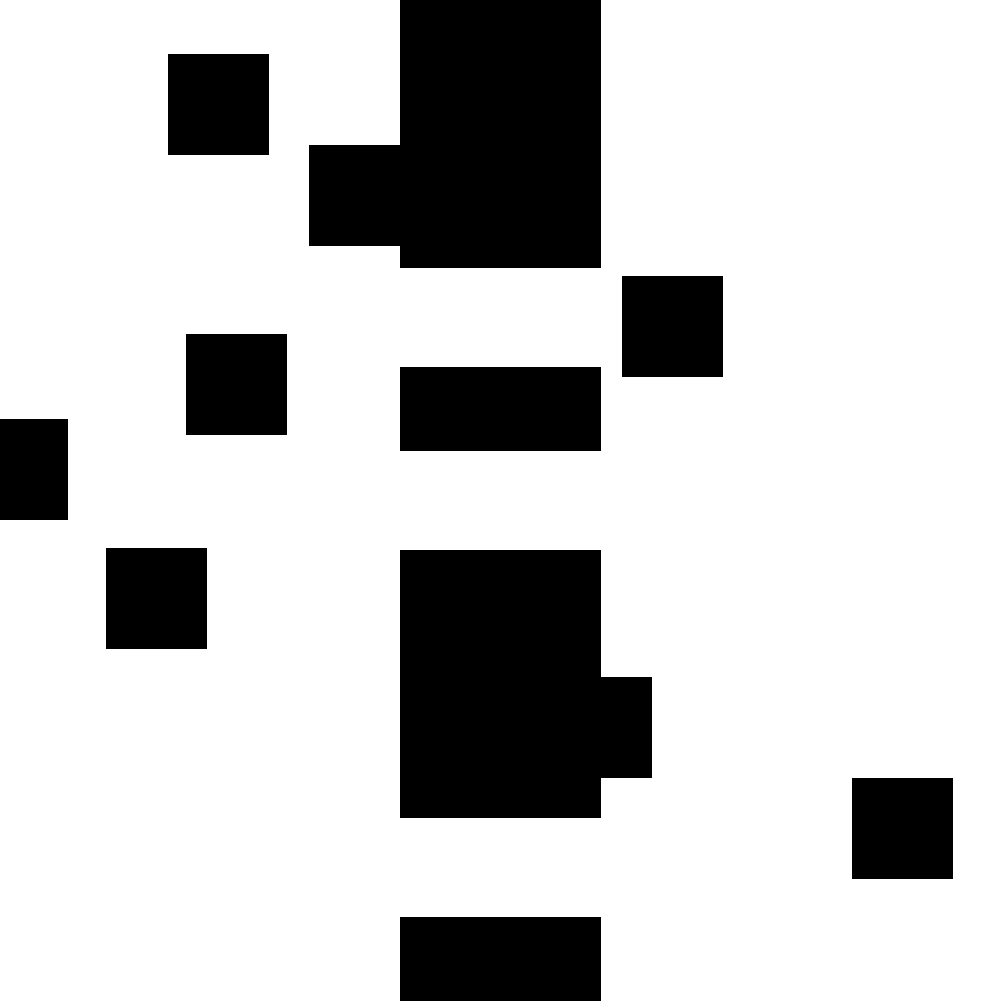}}
\end{subfigure}
\begin{subfigure}{0.9\linewidth}
\hfill
\includegraphics[height=0.3\linewidth]{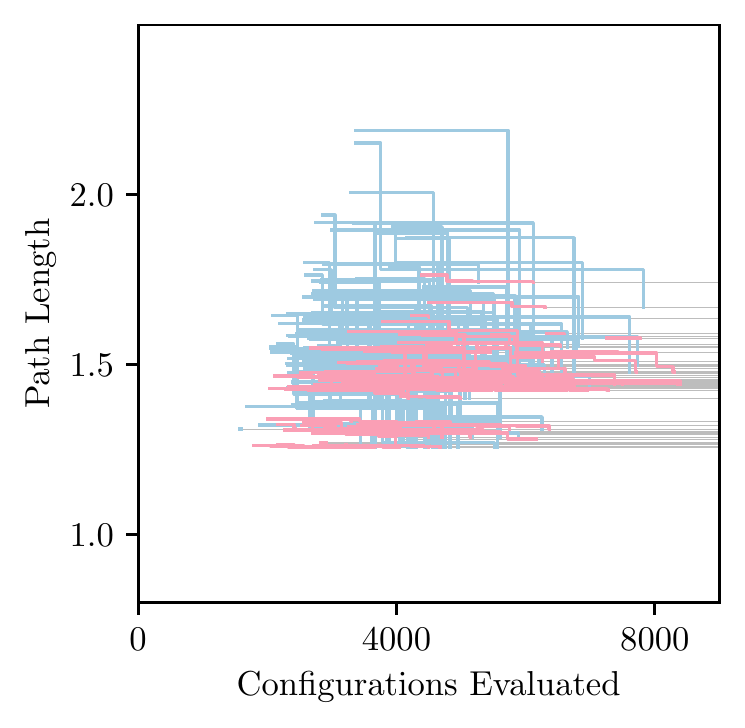}
\hfill
\includegraphics[height=0.3\linewidth]{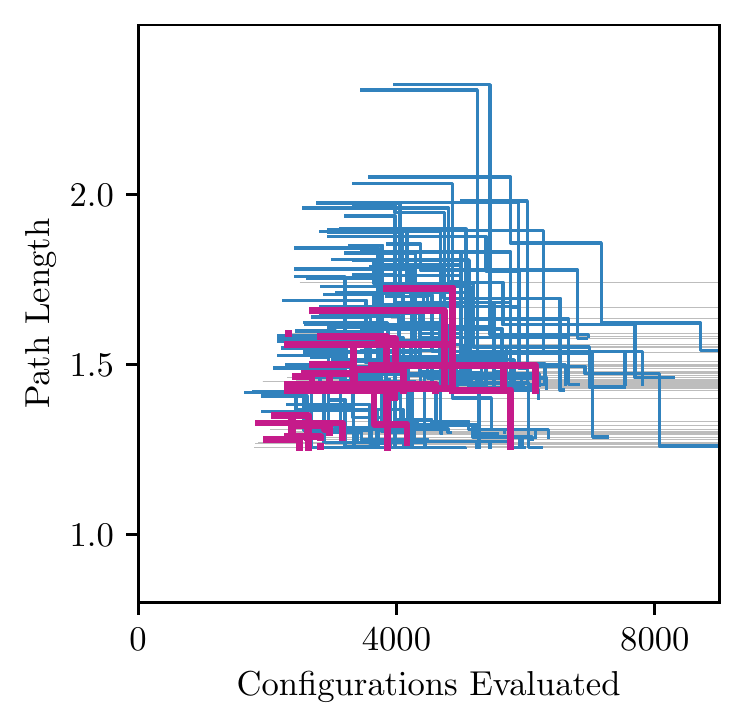}
\hfill
\includegraphics[height=0.3\linewidth]{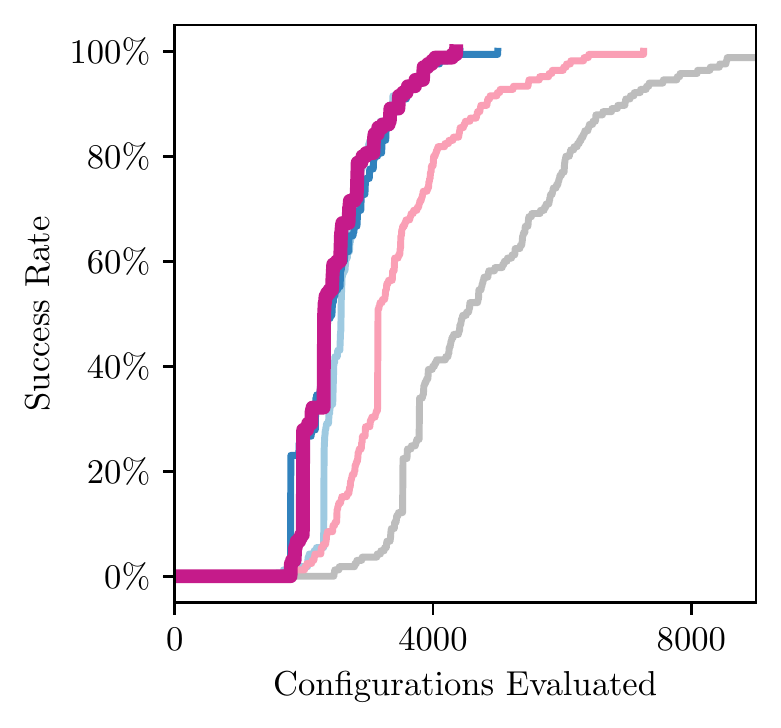}
\end{subfigure}

\vspace{1em}

\begin{subfigure}{0.08\linewidth}
\frame{\includegraphics[height=\linewidth]{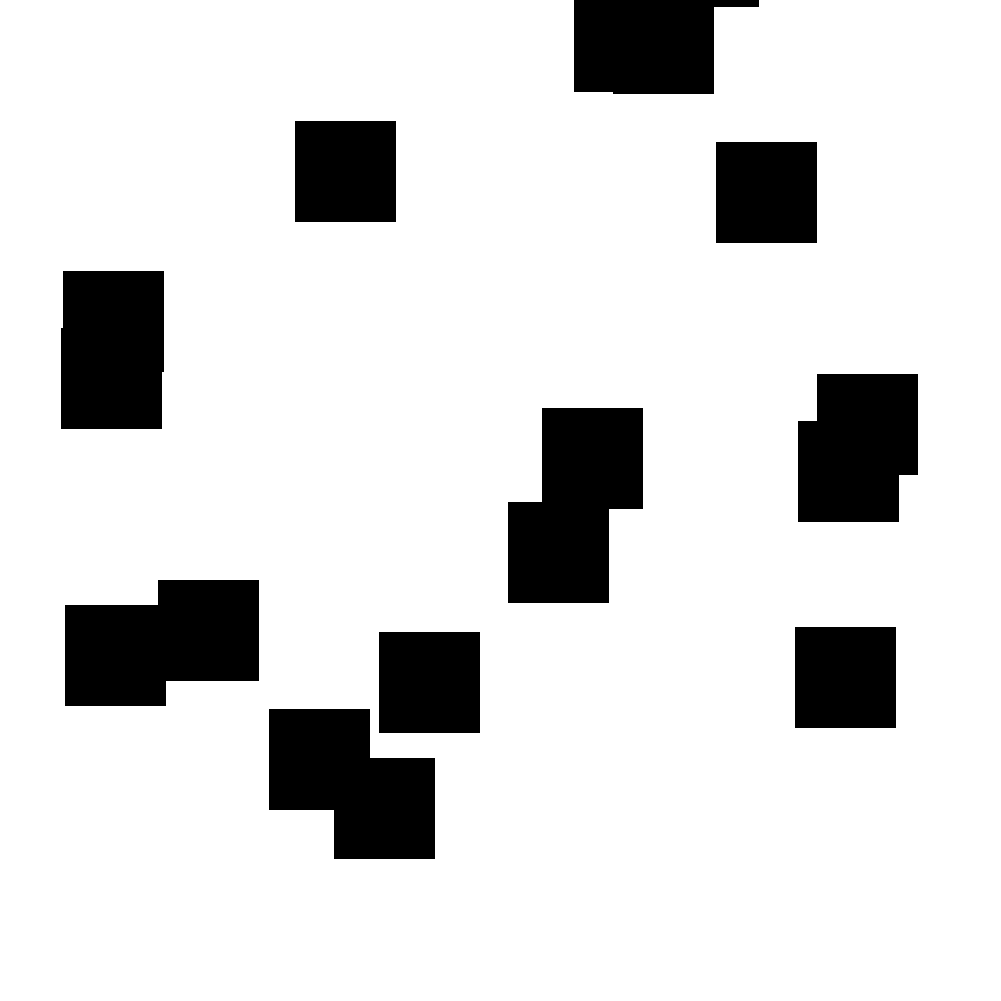}}

\frame{\includegraphics[height=\linewidth]{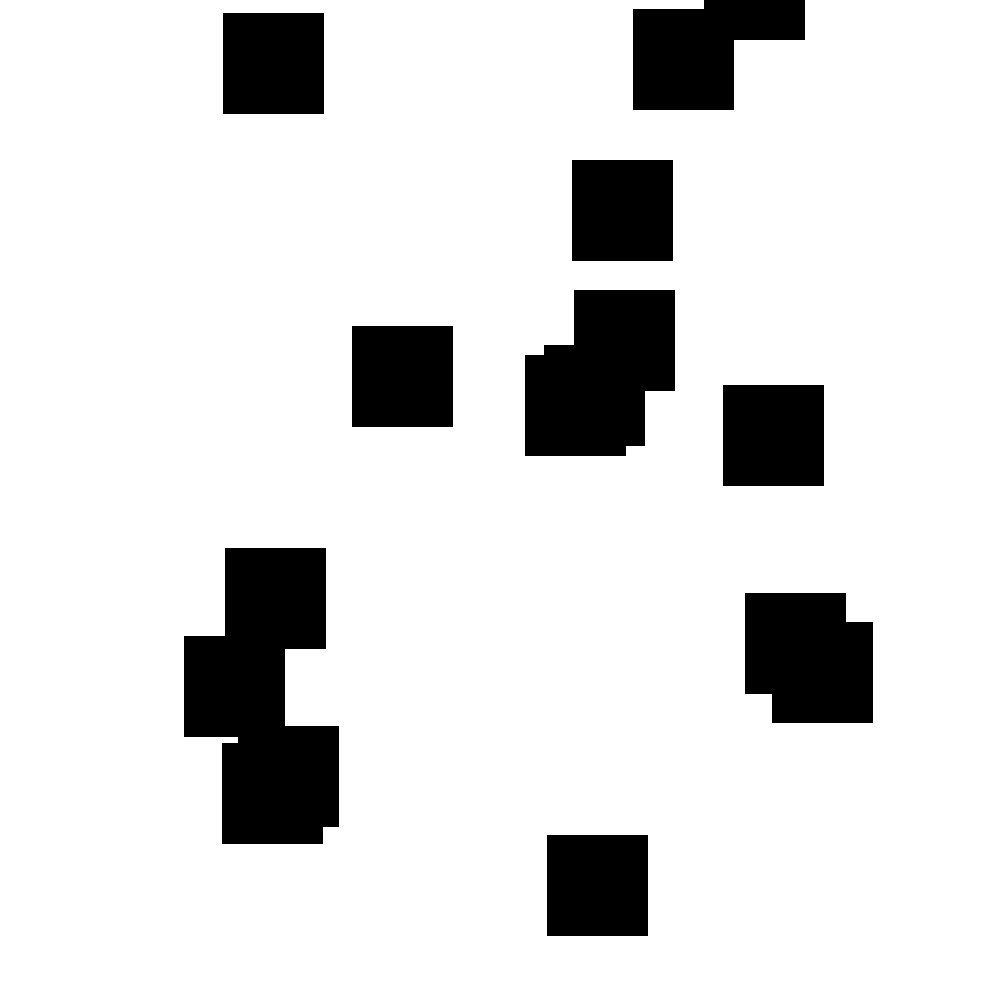}}

\frame{\includegraphics[height=\linewidth]{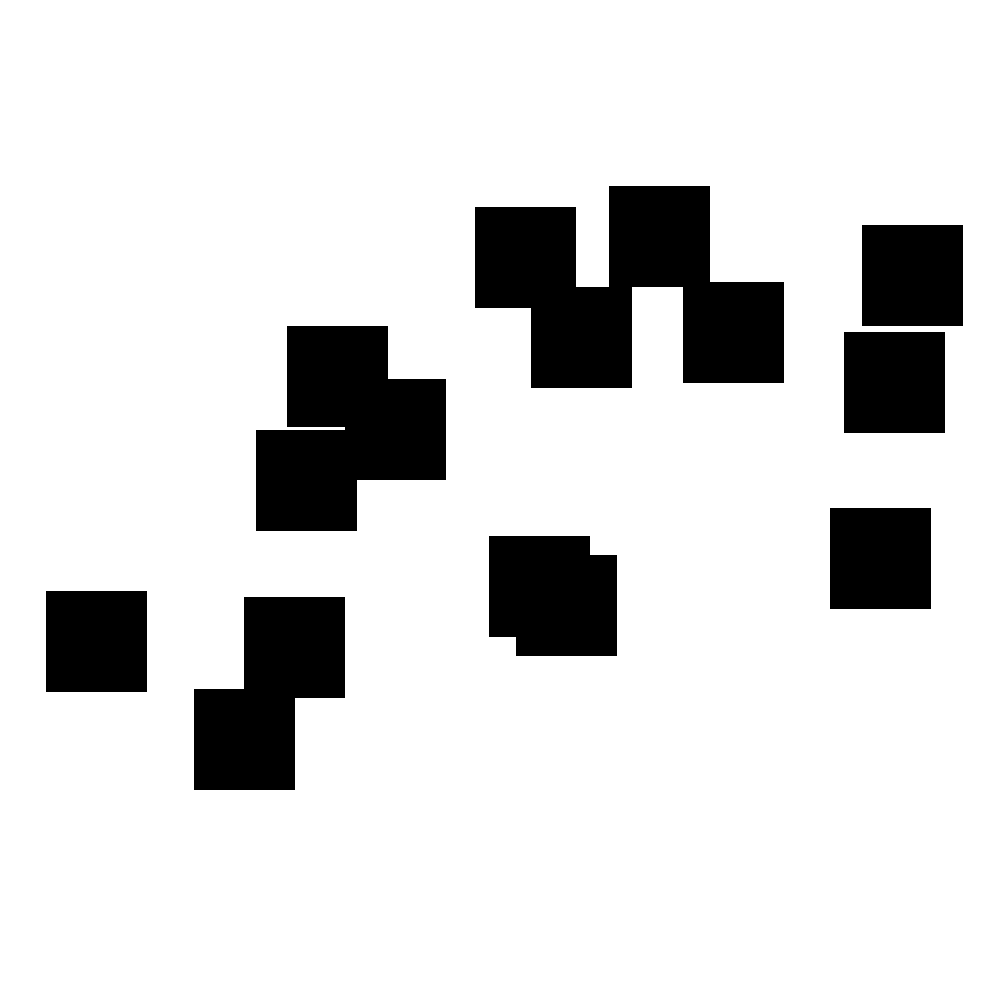}}
\end{subfigure}
\begin{subfigure}{0.9\linewidth}
\hfill
\includegraphics[height=0.3\linewidth]{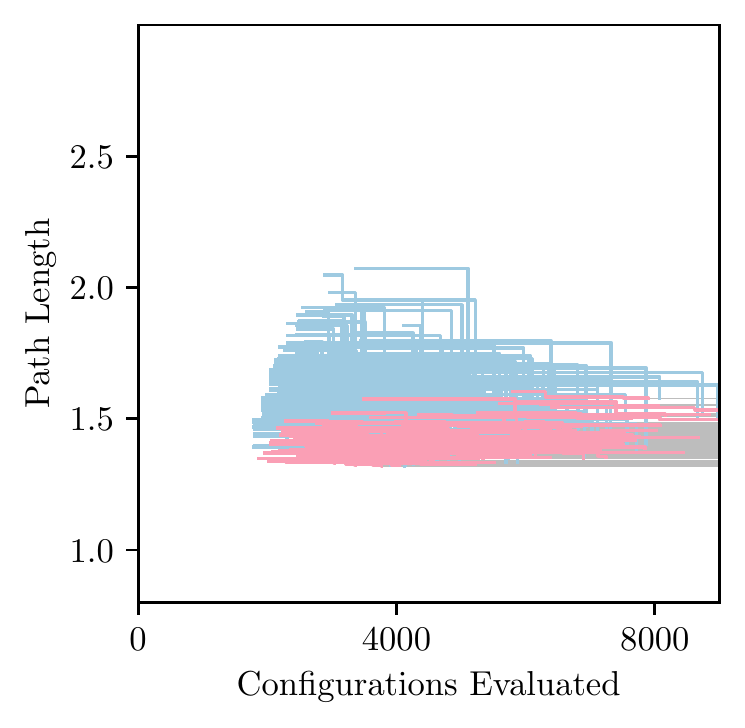}
\hfill
\includegraphics[height=0.3\linewidth]{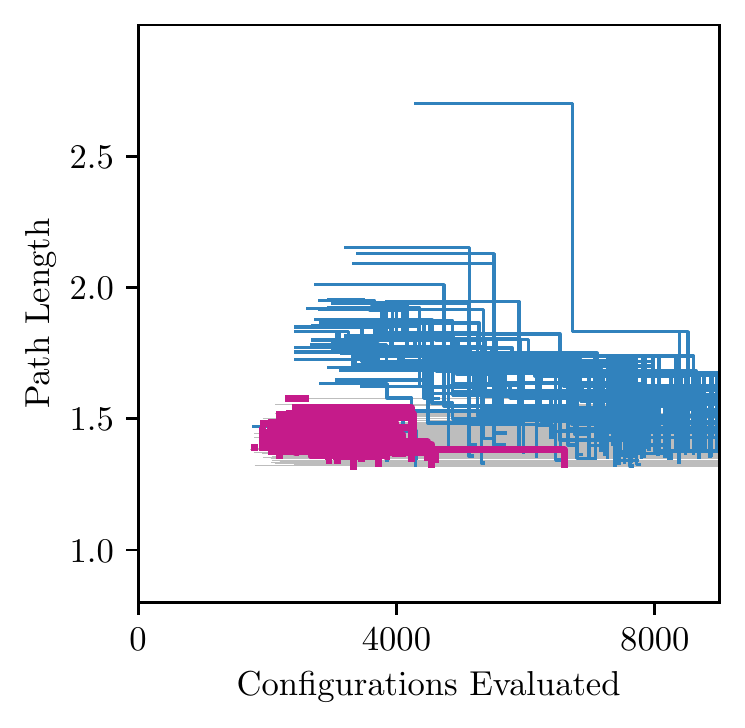}
\hfill
\includegraphics[height=0.3\linewidth]{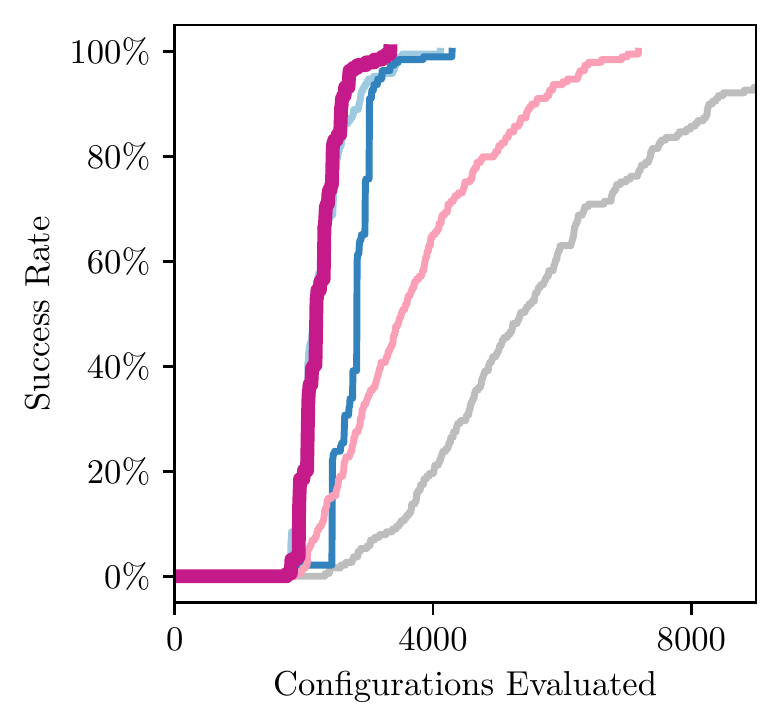}
\end{subfigure}

\vspace{1em}

\includegraphics[width=0.5\linewidth]{figs/legend.pdf}

\caption{
  2-DOF environments.
  (Left) Example problems from each dataset.
  (Center) Length of the best feasible path discovered by each anytime algorithm over time, using the nearest-neighbor posterior (column 2) and finite set posterior (column 3).
  (Right) Collision checking budget versus the percentage of planning problems where that budget is sufficient to discover a feasible path.
}
\label{fig:plots_2d_a}
\end{figure*}

\begin{figure*}
\centering

\begin{subfigure}{0.08\linewidth}
\frame{\includegraphics[height=\linewidth]{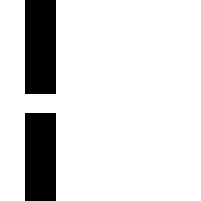}}

\frame{\includegraphics[height=\linewidth]{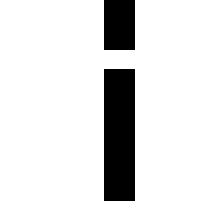}}

\frame{\includegraphics[height=\linewidth]{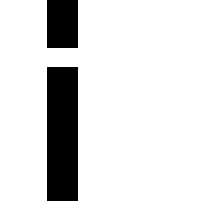}}
\end{subfigure}
\begin{subfigure}{0.9\linewidth}
\hfill
\includegraphics[height=0.3\linewidth]{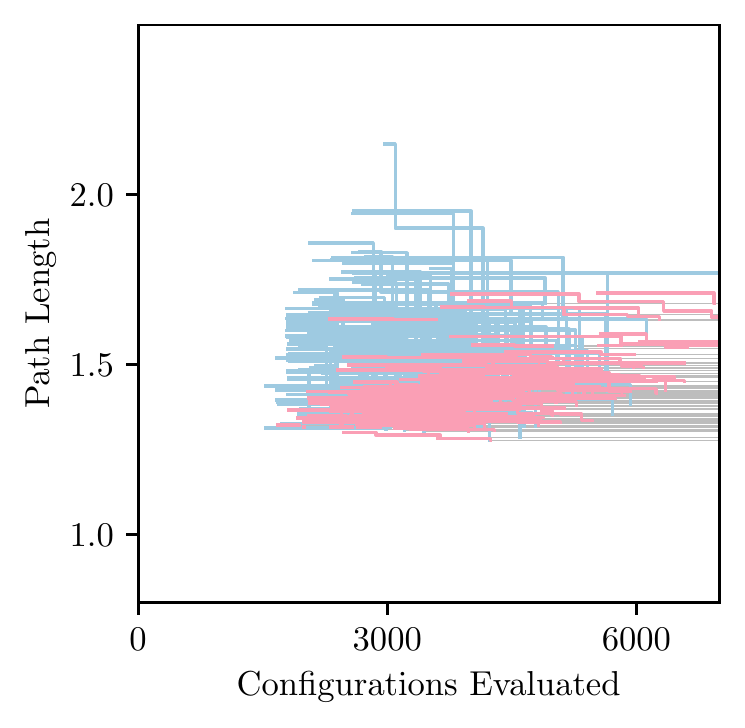}
\hfill
\includegraphics[height=0.3\linewidth]{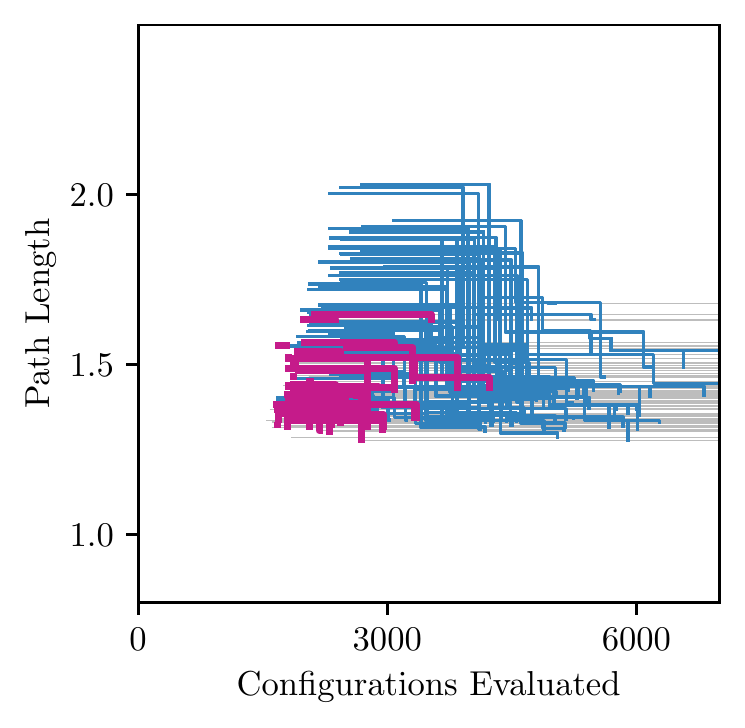}
\hfill
\includegraphics[height=0.3\linewidth]{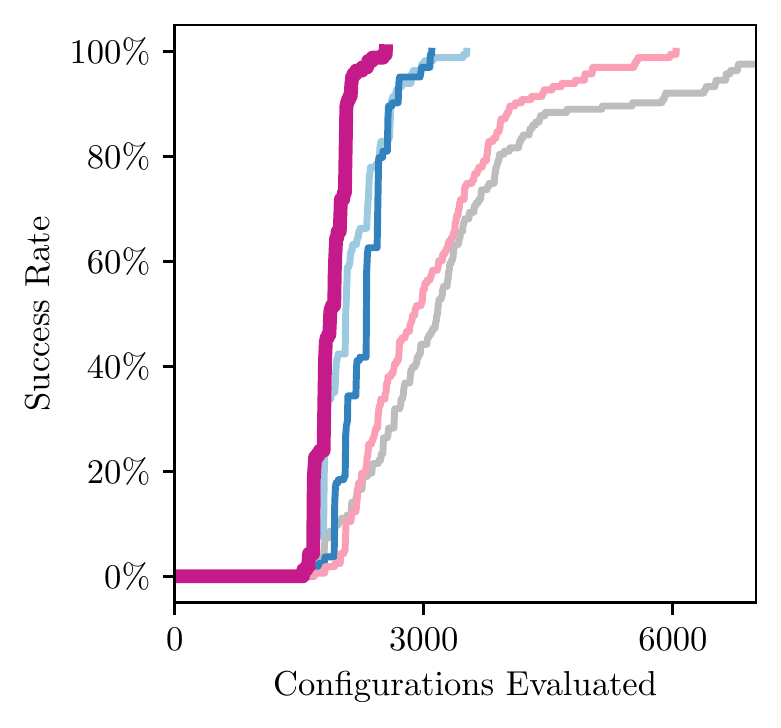}
\end{subfigure}

\vspace{1em}

\begin{subfigure}{0.08\linewidth}
\frame{\includegraphics[height=\linewidth]{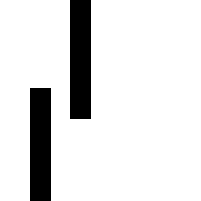}}

\frame{\includegraphics[height=\linewidth]{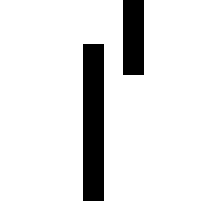}}

\frame{\includegraphics[height=\linewidth]{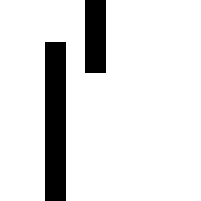}}
\end{subfigure}
\begin{subfigure}{0.9\linewidth}
\hfill
\includegraphics[height=0.3\linewidth]{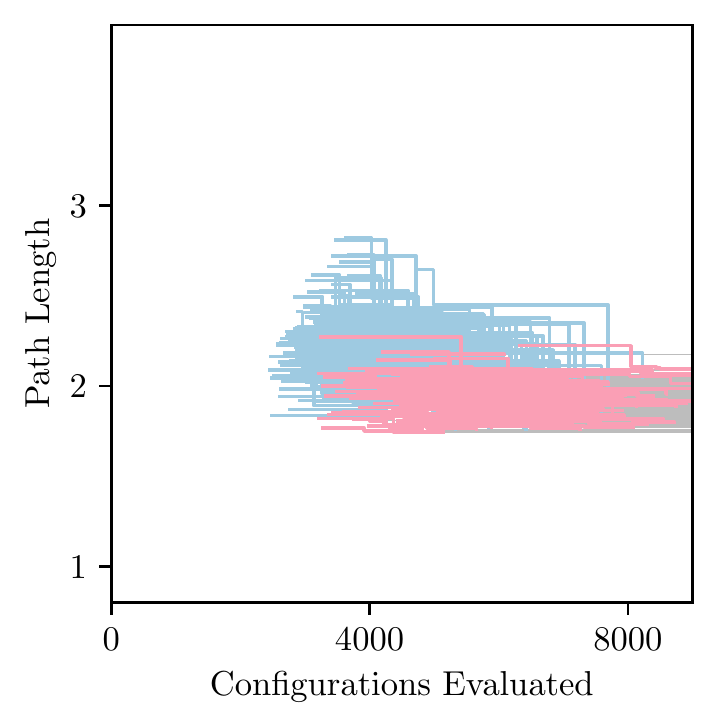}
\hfill
\includegraphics[height=0.3\linewidth]{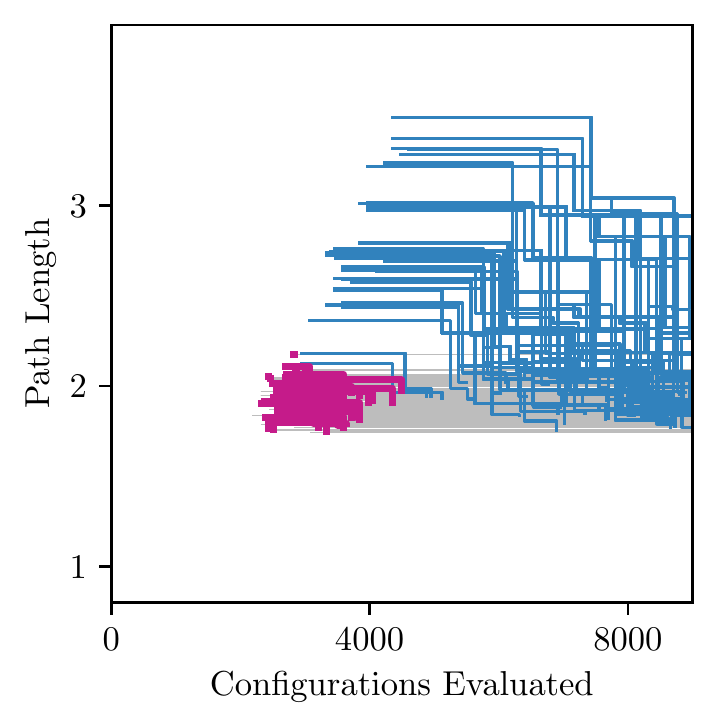}
\hfill
\includegraphics[height=0.3\linewidth]{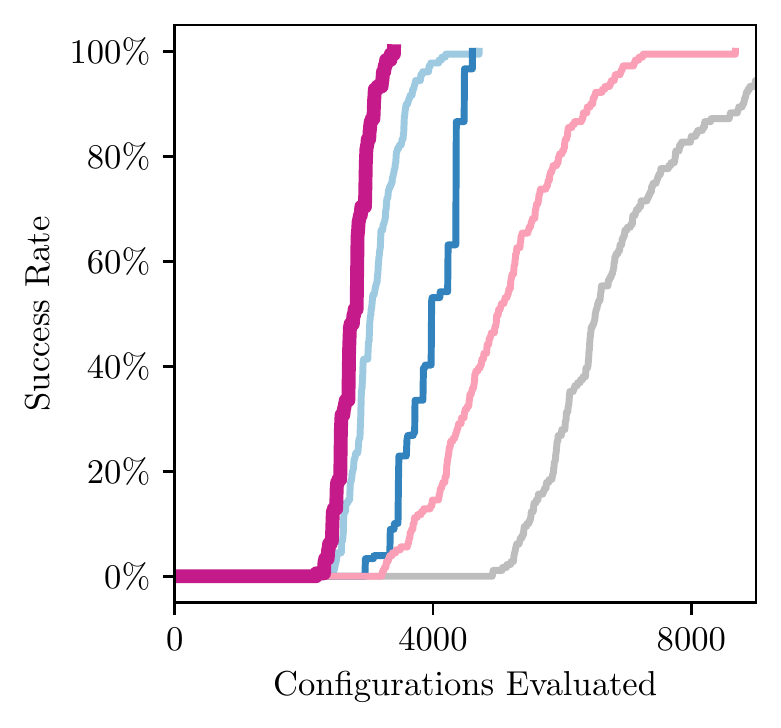}
\end{subfigure}

\vspace{1em}

\begin{subfigure}{0.08\linewidth}
\frame{\includegraphics[height=\linewidth]{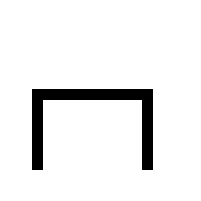}}

\frame{\includegraphics[height=\linewidth]{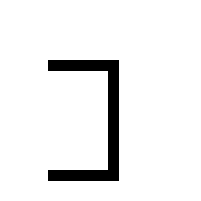}}

\frame{\includegraphics[height=\linewidth]{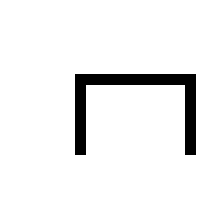}}
\end{subfigure}
\begin{subfigure}{0.9\linewidth}
\hfill
\includegraphics[height=0.3\linewidth]{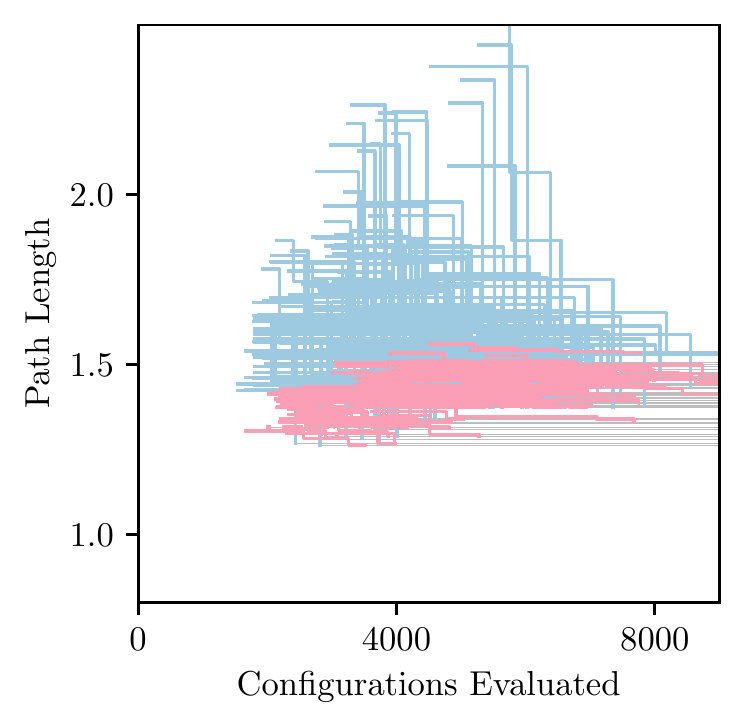}
\hfill
\includegraphics[height=0.3\linewidth]{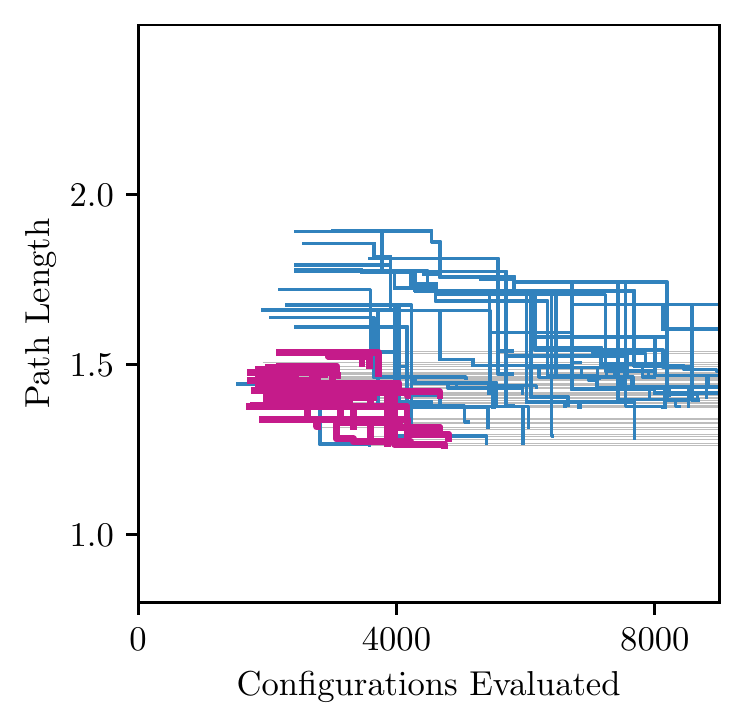}
\hfill
\includegraphics[height=0.3\linewidth]{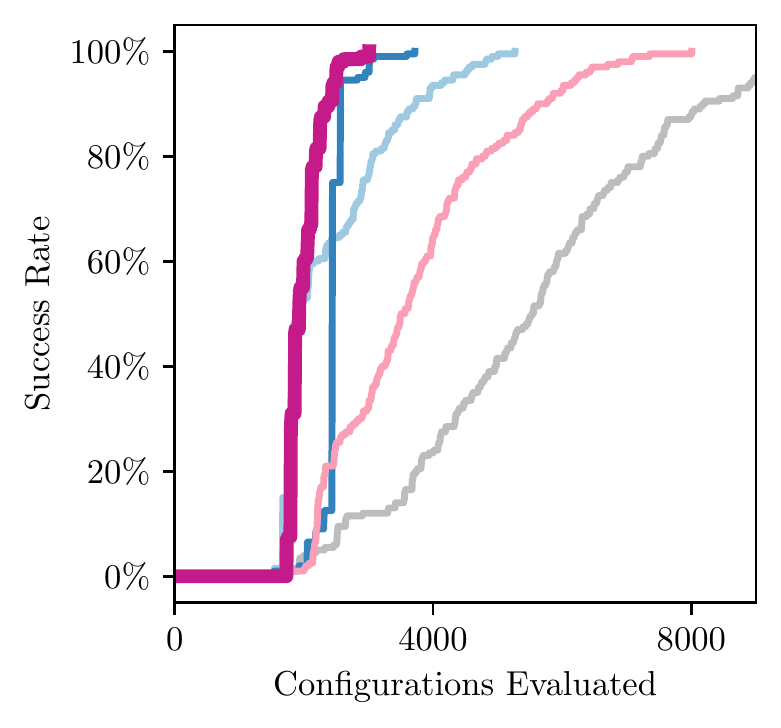}
\end{subfigure}

\vspace{1em}

\includegraphics[width=0.5\linewidth]{figs/legend.pdf}

\caption{
  2-DOF environments.
  (Left) Example problems from each dataset.
  (Center) Length of the best feasible path discovered by each anytime algorithm over time, using the nearest-neighbor posterior (column 2) and finite set posterior (column 3).
  (Right) Collision checking budget versus the percentage of planning problems where that budget is sufficient to discover a feasible path.
}
\label{fig:plots_2d_b}

\end{figure*}